\documentclass{article} 
\usepackage{iclr2019_conference,times}

\usepackage[british]{babel}            
\usepackage[utf8]{inputenc}            
\usepackage[T1]{fontenc}               
\usepackage{hyperref}                  
\usepackage{url}                       
\usepackage{booktabs}                  
\usepackage{amsfonts}                  
\usepackage{nicefrac}                  
\usepackage{microtype}                 
\usepackage{amssymb}		               
\usepackage{xcolor}			               
\usepackage{subcaption}                
\usepackage{graphicx}                  
\usepackage{multirow}		               
\usepackage{comment}		               
\usepackage{tikz}                      
\usepackage{pgfplots}                  
\usepackage[oldenum,olditem]{paralist} 
\usepackage{rotating}

\usepackage{algorithm}                 
\usepackage{algpseudocode}             
\usepackage{amsmath}                   
\usepackage{amsthm}                    
\usepackage{bbm}			                 
\usepackage{chngcntr}                  
\usepackage{dsfont}                    
\usepackage{mathtools}                 
\usepackage{siunitx}                   
\usepackage{tabularx}                  

\pgfplotsset{compat=1.14}

\usepgfplotslibrary{colorbrewer}
\usepgfplotslibrary{statistics}

\usetikzlibrary{external}
\definecolor{printable_1}{RGB}{70,137,102}
\definecolor{printable_2}{RGB}{255,176,59}
\definecolor{printable_3}{RGB}{142,40,0}
\definecolor{printable_4}{RGB}{0,0,0}

\pgfplotscreateplotcyclelist{Firenze}{%
  printable_1,
  printable_2,
  printable_3,
  printable_4
}

\newtheorem{theorem}           {Theorem}

\newtheorem{definition}        {Definition}

\DeclareMathOperator {\neuralpersistence}{NP}                 
\DeclareMathOperator {\normneuralpersistence}{\widetilde{NP}} 
\DeclareMathOperator {\meanneuralpersistence}{\overline{NP}}  
\DeclareMathOperator {\persistence}      {pers}               
\DeclareMathOperator {\weight}           {\varphi}            

\DeclarePairedDelimiter\floor{\lfloor}{\rfloor}

\newcommand{\betti}                  [1]{\ensuremath{\beta_{#1}}}

\newcommand{\diagram}                   {\mathcal{D}}

\newcommand{\norm}                   [1]{\left\lVert#1\right\rVert}

\newcommand{\real}                      {\mathds{R}}
\newcommand{\simplicialcomplex}         {\mathrm{K}}
\newcommand{\weights}                   {\mathcal{W}}

\newif\iffinal

\finaltrue

\title{%
  Neural Persistence: A Complexity Measure for Deep Neural Networks Using Algebraic Topology
}

\author{%
Bastian Rieck$^{1,2, \dagger}$, Matteo Togninalli$^{1,2, \dagger}$, Christian Bock$^{1,2, \dagger}$,\\
\textbf{Michael Moor$^{1,2}$, Max Horn$^{1,2}$, Thomas Gumbsch$^{1,2}$, Karsten Borgwardt$^{1,2}$}\\
\small\textsc{$^{1}$Department of Biosystems Science and Engineering, ETH Zurich, Switzerland}\\
\small\textsc{$^{2}$SIB Swiss Institute of Bioinformatics, Switzerland}\\
\small{$^{\dagger}$These authors contributed equally}
}

\iclrfinalcopy
\begin{document}

\maketitle

\begin{abstract}
  While many approaches to make neural networks more fathomable have been
  proposed, they are restricted to interrogating the network with input
  data.  Measures for characterizing and monitoring structural properties,
  however, have not been developed.  In this work, we propose \emph{neural
  persistence}, a complexity measure for neural network architectures
  based on topological data analysis on weighted stratified graphs.  To
  demonstrate the usefulness of our approach, we show that \emph{neural
  persistence} reflects best practices developed in the deep
  learning community such as dropout and batch normalization.  Moreover,
  we derive a neural persistence-based stopping criterion that shortens
  the training process while achieving comparable accuracies as early
  stopping based on validation loss.
\end{abstract}

\hypersetup{%
  pdftitle    = {%
    Neural Persistence: A Complexity Measure for Deep Neural Networks Using Algebraic Topology
  },
  pdfauthor   = {%
    Bastian Rieck, Matteo Togninalli, Christian Bock, Michael Moor, Max
    Horn, Thomas Gumbsch, and Karsten Borwardt
},
  pdfkeywords = {algebraic topology, deep learning, persistent homology, early stopping, complexity},
}

\section{Introduction}

The practical successes of deep learning in various fields such as
image processing~\citep{simonyan2014very,he2016deep,hu2017squeeze},
biomedicine~\citep{ching2018opportunities, rajpurkar2017chexnet, rajkomar2018scalable},
and language translation~\citep{bahdanau2014neural, sutskever2014sequence,wu2016google}
still outpace our theoretical understanding.
While hyperparameter adjustment strategies exist~\citep{Bengio12}, formal
measures for assessing the generalization capabilities of deep neural
networks have yet to be identified~\citep{zhang2016understanding}.
Previous approaches for improving theoretical and practical comprehension
focus on interrogating networks with input data. These methods
include
\begin{inparaenum}[i)]
  \item feature visualization of deep convolutional neural networks~\citep{zeiler2014visualizing, springenberg2014striving},
  \item sensitivity and relevance analysis of features~\citep{montavon2017methods},
  \item a descriptive analysis of the training process based on information theory~\citep{tishby2015deep, shwartz2017opening, saxe2018information, achille2017emergence}, and
  \item a statistical analysis of interactions of the learned weights~\citep{tsang2017detecting}.
\end{inparaenum}
Additionally, \citet{raghu2016expressive} develop
a measure of \emph{expressivity} of a neural network and use it to
explore the empirical success of batch normalization, as well as for the
definition of a new regularization method. They note that one key
challenge remains, namely to provide meaningful insights while maintaining
theoretical generality. 
This paper presents a method for elucidating neural networks in light of
both aspects.

We develop \emph{neural persistence}, a novel measure for characterizing
neural network structural complexity.
In doing so, we adopt a new perspective that integrates both network
weights and connectivity while not relying on interrogating networks
through input data.
Neural persistence builds on computational techniques from algebraic topology,
specifically topological data analysis~(TDA), which was already shown to be
beneficial for feature extraction in deep learning~\citep{hofer2017deep} and
describing the complexity of GAN sample spaces~\citep{Khrulkov18}.
More precisely, we rephrase deep networks with fully-connected layers
into the language of algebraic topology and develop a measure for
assessing the structural complexity of
\begin{inparaenum}[i)]
  \item individual layers, and
  \item the entire network.
\end{inparaenum}
In this work, we present the following contributions:
\begin{compactitem}[-]
  \item We introduce \emph{neural persistence}, a novel measure for
  characterizing the structural complexity of neural networks that can
  be efficiently computed.
  \item We prove its theoretical properties, such as upper and lower
  bounds, thereby arriving at a normalization for comparing neural
  networks of varying sizes.
  \item We demonstrate the practical utility of neural persistence in
  two scenarios:
  \begin{inparaenum}[i)]
    \item it correctly captures the benefits of dropout and batch
    normalization during the training process, and
    \item it can be easily used as a competitive early stopping criterion
    that does not require validation data.
  \end{inparaenum}
\end{compactitem}

\section{Background: Topological data analysis}\label{sec:background}

Topological data analysis~(TDA) recently emerged as a field that provides computational tools for
analysing complex data within a rigorous mathematical framework that is
based on \emph{algebraic topology}.
This paper uses persistent homology, a theory that was developed
to understand high-dimensional manifolds~\citep{Edelsbrunner02, Edelsbrunner10},
and has since been successfully employed in characterizing graphs~\citep{Sizemore17, Rieck18a},
finding relevant features in unstructured data~\citep{Lum13}, and analysing image manifolds~\citep{Carlsson08}.
This section gives a brief summary of the key concepts; please refer to
\citet{Edelsbrunner10} for an extensive introduction.

\paragraph{Simplicial homology}
%
The central object in algebraic topology is a simplicial
complex~$\simplicialcomplex$, i.e.\ a high-dimensional generalization of
a graph, which is typically used to describe complex objects such as
manifolds.
Various notions to describe the connectivity of $\simplicialcomplex$
exist, one of them being simplicial homology.
Briefly put, simplicial homology uses matrix reduction
algorithms~\citep{munkres2018elements} to derive a set of groups, the
homology groups, for a given simplicial complex $\simplicialcomplex$.
Homology groups describe topological features---colloquially also
referred to as holes---of a certain dimension~$d$, such as connected
components~($d=0$), tunnels~($d=1$), and voids~($d=2$).
The information from the $d$th homology group is summarized in a simple
complexity measure, the $d$th Betti number~$\betti{d}$, which merely
counts the number of $d$-dimensional features:
a circle, for example, has Betti numbers $(1,1)$, i.e.\ one connected
component and one tunnel, while a filled circle has Betti numbers
$(1,0)$, i.e.\ one connected component but no tunnel.
In the context of analysing simple feedforward neural networks for two
classes, \citet{Bianchini14} calculated bounds of Betti numbers of the
decision region belonging to the positive class, and were thus able to
show the implications of different activation functions.
These ideas were extended by \citet{guss2018characterizing} to obtain
a measure of the topological complexity of decision boundaries.

\paragraph{Persistent homology}
%
For the analysis of real-world data sets, however, Betti numbers turn
out to be of limited use because their representation is too coarse and
unstable.
This prompted the development of persistent homology.
Given a simplicial complex $\simplicialcomplex$ with an additional set
of weights $a_0 \leq a_1 \leq \dots \leq a_{m-1} \leq a_m$, which are
commonly thought to represent the idea of a scale, it is possible to put
$\simplicialcomplex$ in a filtration, i.e.\ a nested sequence of
simplicial complexes
$\emptyset = \simplicialcomplex_0 \subseteq \simplicialcomplex_1 \subseteq \dots \subseteq \simplicialcomplex_{m-1} \subseteq \simplicialcomplex_m = \simplicialcomplex$.
This filtration is thought to represent the `growth' of
$\simplicialcomplex$ as the scale is being changed.
During this growth process, topological features can be \emph{created}~(new
vertices may be added, for example, which creates a new connected component)
or \emph{destroyed}~(two connected components may merge into one).
Persistent homology tracks these changes and represents the creation and
destruction of a feature as a point $(a_i, a_j) \in \real^2$
for indices $i \leq j$ with respect to the filtration.
The collection of all points corresponding to $d$-dimensional
topological features is called the $d$th persistence
diagram~$\diagram_d$. It can be seen as a collection of Betti numbers
at multiple scales.
Given a point $(x,y) \in \diagram_d$, the quantity $\persistence(x,y) :=
|y-x|$ is referred to as its \emph{persistence}.
Typically, high persistence is considered to correspond to
features, while low persistence is considered to indicate
noise~\citep{Edelsbrunner02}.
  
\section{A novel measure for neural network complexity}

\begin{figure}[tbp]
  \iffinal
    \includegraphics{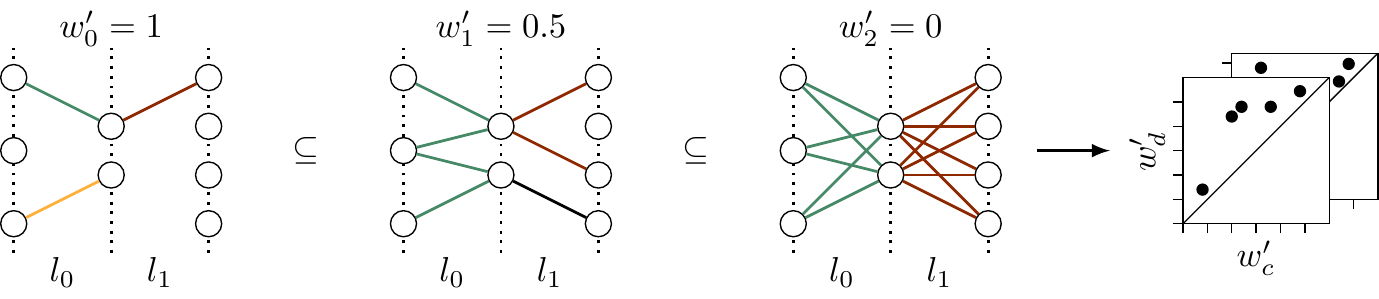}
  \else
    \input{figures/tikz/figure0.tex}
  \fi
	\caption{%
    Illustrating the neural persistence calculation of a network with
    two layers ($l_0$ and $l_1$).
    Colours indicate connected components per layer.
    The filtration process is depicted by colouring connected
    components that are created or merged when the respective
    weights are greater than or equal to the threshold $w'_i$.
    As $w'_i$ decreases, network connectivity increases.
    Creation and destruction thresholds are collected in one persistence
    diagram per layer~(right), and summarized according to
    Equation~\ref{eq:p-norm} for calculating neural persistence.
  }
	\label{fig:filtration}
\end{figure}

This section details \emph{neural persistence}, our novel measure for
assessing the structural complexity of neural networks. By exploiting
both network structure and weight information through persistent
homology, our measure captures network expressiveness and goes beyond
mere connectivity properties.
Subsequently, we describe its calculation, provide theorems for
theoretical and empirical bounds, and show the existence of neural
networks complexity regimes.
To summarize this section, Figure~\ref{fig:filtration} illustrates
how our method treats a neural network.

\subsection{Neural persistence}\label{sec:Neural persistence}

Given a feedforward neural network with an arrangement of neurons and
their connections~$E$, let $\weights$ refer to the set of weights.
Since $\weights$ is typically changing during training, we require a
function $\weight\colon{}E\to\weights$ that maps a specific edge
to a weight.
Fixing an activation function, the connections form a \emph{stratified
graph}.
\begin{definition}[Stratified graph and layers]\label{def:Stratified}
  A \emph{stratified graph} is a multipartite graph $G = (V,E)$ satisfying \ $V = V_0
  \sqcup V_1 \sqcup \dots$, such that if $u \in V_i$, $v \in V_j$, and
  $(u,v) \in E$, we have $j = i+1$. Hence, edges are only permitted between
  adjacent vertex sets. Given $k \in \mathds{N}$, the $k$th \emph{layer}
  of a stratified graph is the unique subgraph
  $G_k := (V_k \sqcup V_{k+1}, E_k := E \cap \{ V_k \times V_{k+1} \})$.
  \label{def:Layer}
\end{definition}
This enables calculating the persistent homology of $G$ and each $G_k$,
using the filtration induced by sorting all weights,
which is common practice in topology-based network
analysis~\citep{Carstens13, Horak09} where weights often represent
closeness or node similarity.
However, our context requires a novel filtration because the weights
arise from an incremental fitting procedure, namely the training, which
could theoretically lead to unbounded values.
When analysing geometrical data with persistent homology, one typically
selects a filtration based on the~(Euclidean) distance between data
points~\citep{Bubenik15}. The filtration then connects points that are
increasingly distant from each other, starting from points that are
direct neighbours. Our network filtration aims to mimic this behaviour
in the context of fully-connected neural networks.
Our framework does not \emph{explicitly} take activation functions into
account; however, activation functions influence the evolution
of weights during training.

\paragraph{Filtration}
%
Given the set of weights~$\weights$ for one training step, let $w_{\max}
:= \max_{w \in \weights} |w|$.
Furthermore, let $\weights' := \{ |w| / w_{\max} \mid w \in \weights \}$
be the set of transformed weights, indexed in non-ascending order, such
that $1 = w'_0 \geq w'_1 \geq \dots \geq 0$.
This permits us to define a filtration for the $k$th layer $G_k$ as $G_k^{(0)}
\subseteq G_k^{(1)} \subseteq \dots$, where
$G_k^{(i)} := \left(V_k \sqcup V_{k+1}, \{ (u,v) \mid (u,v) \in E_k \wedge \weight'(u,v) \geq w'_i \} \right)$
and $\weight'(u,v) \in \weights'$ denotes the transformed weight of an edge.
We tailored this filtration towards the analysis of neural networks, for
which large~(absolute) weights indicate that certain neurons exert
a larger influence over the final activation of a layer.
The strength of a connection is thus preserved by the filtration, and
weaker weights with $|w| \approx 0$ remain close to $0$. Moreover, since
$w' \in [0,1]$ holds for the transformed weights, this filtration makes the
network invariant to scaling, which simplifies the comparison of different
networks.

\paragraph{Persistence diagrams}
%
Having set up the filtration, we can calculate persistent homology for
every layer $G_k$.
As the filtration contains at most $1$-simplices~(edges), we capture
zero-dimensional topological information, i.e.\ how connected
components are created and merged during the filtration.
These information are structurally equivalent to calculating a maximum
spanning tree using the weights, or performing hierarchical clustering
with a specific setup~\citep{Carlsson10}.
While it would theoretically be possible to include higher-dimensional
information about each layer $G_k$, for example in the form of
cliques~\citep{Rieck18a}, we focus on zero-dimensional information in
this paper, because of the following advantages:
\begin{inparaenum}[i)]
  \item the resulting values are easily interpretable as they
  essentially describe the clustering of the network at multiple
  weight thresholds,
  \item previous research~\citep{Rieck16a, hofer2017deep} indicates that
  zero-dimensional topological information is already capturing a large
  amount of information, and
  \item persistent homology calculations are highly efficient in this regime~(see below).
\end{inparaenum}
We thus calculate zero-dimensional persistent homology with this
filtration. The resulting persistence diagrams have a special structure:
since our filtration solely sorts \emph{edges}, all vertices are present
at the beginning of the filtration, i.e.\ they are already part of
$G_k^{(0)}$ for each $k$. As a consequence, they are assigned a weight
of $1$, resulting in $|V_k \times V_{k+1}|$ connected components.
Hence, entries in the corresponding persistence diagram $\diagram_k$ are
of the form $(1,x)$, with $x \in \weights'$, and will be situated
\emph{below} the diagonal, similar to superlevel set
filtrations~\citep{Bubenik15,Cohen-Steiner09}.
Using the $p$-norm of a persistence diagram, as introduced by
\citet{Cohen-Steiner10}, we obtain the following definition for neural
persistence.
%
\begin{definition}[Neural persistence]
The neural persistence of the $k$th layer $G_k$, denoted by
$\neuralpersistence(G_k)$, is the $p$-norm of the persistence diagram
$\diagram_k$ resulting from our previously-introduced filtration, i.e.\
\begin{equation}
  \neuralpersistence(G_k) := \norm{\diagram_k}_p := \Big(\sum_{(c,d) \in \diagram_k} \persistence(c,d)^p\Big)^\frac{1}{p},
  \label{eq:p-norm}
\end{equation}
which~(for $p=2$) captures the Euclidean distance of points in
$\diagram_k$ to the diagonal.
\end{definition}

The $p$-norm is known to be a stable summary~\citep{Cohen-Steiner10} of
topological features in a persistence diagram.
For neural persistence to be a meaningful measure of structural
complexity, it should increase as a neural network is learning.
We evaluate this and other properties in Section~\ref{sec:Experiments}.

Algorithm~\ref{alg:Neural persistence} provides pseudocode for the
calculation process. It is highly efficient: the filtration~(line~\ref{pcline:filtration})
amounts to sorting all $n$ weights of a network, which has
a computational complexity of $\mathcal{O}(n \log{} n)$.
Calculating persistent homology of this filtration~(line~\ref{pcline:pers}) can
be realized using an algorithm based on union--find data structures~\cite{Edelsbrunner02}.
This has a computational complexity of $\mathcal{O}\left(n \cdot \alpha\left(n\right)\right)$,
where $\alpha(\cdot)$ refers to the extremely slow-growing inverse of
the Ackermann function~\citep[Chapter~22]{Cormen09}.
We make our implementation and experiments available under
\url{https://github.com/BorgwardtLab/Neural-Persistence}.

\begin{algorithm}[bt]
  \caption{Neural persistence calculation}
  \label{alg:Neural persistence}
  \algorithmicrequire{} Neural network with $l$ layers and weights~$\weights$
  \begin{algorithmic}[1]
    \State $w_{\max} \gets \max_{w \in \weights} |w|$                     \Comment{Determine largest absolute weight}

    \State $\weights' \gets \{ |w| / w_{\max} \mid w \in \weights \}$           \Comment{Transform weights for filtration} \label{pcline:init}
    \For{$k \in \{0,\dots,l-1\}$} \label{pcline:forloop}
      \State $F_k \gets G_k^{(0)} \subseteq G_k^{(1)} \subseteq \dots$   \Comment{Establish filtration of $k$th layer} \label{pcline:filtration}
      \State $\diagram_k \gets \Call{PersistentHomology}{F_k}$  \Comment{Calculate persistence diagram} \label{pcline:pers}
    \EndFor
    \State \textbf{return} $\{\norm{\diagram_0}_p, \dots, \norm{\diagram_{l-1}}_p \}$ \Comment{Calculate neural persistence for each layer}
  \end{algorithmic}
\end{algorithm}

\subsection{Properties of neural persistence}\label{sec:Theory}

We elucidate properties about neural persistence to permit the
comparison of networks with different architectures.  As a first step,
we derive \emph{bounds} for the neural persistence of a single layer
$G_k$.

\begin{theorem}
  \label{thm:Neural persistence bounds}
  Let $G_k$ be a layer of a neural network according to
  Definition~\ref{def:Layer}.
  Furthermore, let $\weight_k\colon{}E_k\to\weights'$ denote the
  function that assigns each edge of $G_k$ a transformed weight.
  Using the filtration from Section~\ref{sec:Neural persistence} to
  calculate persistent homology, the neural persistence $\neuralpersistence(G_k)$
  of the $k$th layer satisfies
  \begin{equation}
    0 \leq \neuralpersistence(G_k) \leq \left( \max_{e\in E_k} \weight_k(e) - \min_{e\in E_k} \weight_k(e) \right) ( |V_k \times V_{k+1}|  - 1 )^\frac{1}{p},
    \label{eq:Total persistence bounds}
  \end{equation}
  where $|V_k \times V_{k+1}|$ denotes the cardinality of the vertex set, i.e.\ the
  number of neurons in the layer.
\end{theorem}
\begin{proof}
  We prove this constructively and show that the bounds can be
  realized. For the lower bound, let $G_k^-$ be a fully-connected
  layer with $|V_k|$ vertices and, given $\theta \in [0,1]$, let
  $\weight_k(e) := \theta$ for every edge $e$.
  Since a vertex $v$ is created before its incident edges, the
  filtration degenerates to a lexicographical ordering of vertices and
  edges, and all points in $\diagram_k$ will be of the form
  $(\theta,\theta)$. Thus, $\neuralpersistence(G_k^-) = 0$.
  For the upper bound, let $G_k^+$ again be
  a fully-connected layer with $|V_k| \geq 3$ vertices and let $a, b \in
  [0,1]$ with $a < b$.
  Select one edge $e'$ at random and define a weight function as
  $\weight(e') := b$ and $\weight(e) := a$ otherwise.
  In the filtration, the addition of the first edge will create a pair
  of the form $(b,b)$, while all other pairs will be of the form
  $(b,a)$. Consequently, we have
  \begin{align}
    \neuralpersistence(G_k^+) &= \Big(\persistence(b,b)^p + (n-1)\cdot\persistence(b,a)^p\Big)^\frac{1}{p} = (b-a)\cdot(n-1)^\frac{1}{p}\label{eq:Upper bound difference}\\
                            &= \left( \max_{e\in E_k} \weight(e) - \min_{e\in E_k} \weight(e) \right) ( |V_k| - 1 )^\frac{1}{p},
  \end{align}
  so our upper bound can be realized.
  To show that this term cannot be exceeded by $\neuralpersistence(G)$
  for any $G$, suppose we perturb the weight function
  $\widetilde{\weight}(e) := \weight(e) + \epsilon\in [0,1]$.
  This cannot increase $\neuralpersistence$, however, because each
  difference $b-a$ in Equation~\ref{eq:Upper bound
  difference} is maximized by $\max\weight(e) - \min\weight(e)$.
\end{proof}

We can use the upper bound of Theorem~\ref{thm:Neural persistence
bounds} to normalize the neural persistence of a layer, making it
possible to compare layers~(and neural networks) that feature different
architectures, i.e.\ a different number of neurons.
\begin{definition}[Normalized neural persistence]
  For a layer $G_k$ following Definition~\ref{def:Layer}, using the
  upper bound of Theorem~\ref{thm:Neural persistence bounds}, the
  \emph{normalized neural persistence} $\normneuralpersistence(G_k)$ is
  defined as the neural persistence of $G_k$ divided by its upper bound,
  i.e.\
  $\normneuralpersistence(G_k) := \neuralpersistence(G_k) \cdot \neuralpersistence(G_k^+)^{-1}$.
\end{definition}
The normalized neural persistence of a layer permits us to extend the
definition to an entire network. While this is more complex than using
a single filtration for a neural network, this permits us to side-step
the problem of different layers having different scales.
\begin{definition}[Mean normalized neural persistence]
  Considering a network as a stratified graph $G$ according to
  Definition~\ref{def:Stratified}, we sum the neural persistence values
  per layer to obtain the \emph{mean normalized neural persistence},
  i.e.\
	$\meanneuralpersistence(G) := 1/l \cdot \sum_{k=0}^{l-1} \normneuralpersistence(G_k)$.
  \label{def:mean-normalized-np}
\end{definition}
While Theorem~\ref{thm:Neural persistence bounds} gives a lower and
upper bound in a general setting, it is possible to obtain
empirical bounds when we consider the tuples that result from the
computation of a persistence diagram.
Recall that our filtration ensures that the persistence diagram of a
layer contains tuples of the form $(1,w_i)$, with $w_i \in [0,1]$ being a
transformed weight.
Exploiting this structure permits us to obtain bounds that could be used
prior to calculating the actual neural persistence value in order to
make the implementation more efficient.
%
\begin{theorem}
	\label{thm:Tighter bounds}
  Let $G_k$ be a layer of a neural network as in Theorem~\ref{thm:Neural
  persistence bounds} with $n$ vertices and $m$ edges whose edge weights
  are sorted in non-descending order, i.e.\ $w_0 \leq w_2 \leq \dots
  \leq w_{m-1}$.
  Then $\neuralpersistence(G_k)$ can be empirically bounded by
	\begin{equation}
    \norm{\mathds{1} - \mathbf{w}_{\max}}_p \leq \neuralpersistence(G_k) \leq \norm{\mathds{1} - \mathbf{w}_{\min}}_p,
	  \label{eq:Tighter bounds}
	\end{equation}
	where
  $\mathbf{w}_{\max} = \left( w_{m-1}, w_{m-2}, \dots, w_{m-n}\right)^T$
  and
  $\mathbf{w}_{\min} = \left( w_0, w_2, \dots, w_{n-1}\right)^T$
  are the vectors containing the $n$ largest and $n$ smallest weights,
  respectively.
\end{theorem}
\begin{proof}\renewcommand{\qedsymbol}{}
  See Section~\ref{app:Proofs} in the appendix.
\end{proof}

\begin{figure}[tbp]
  \centering
    \iffinal
      \includegraphics{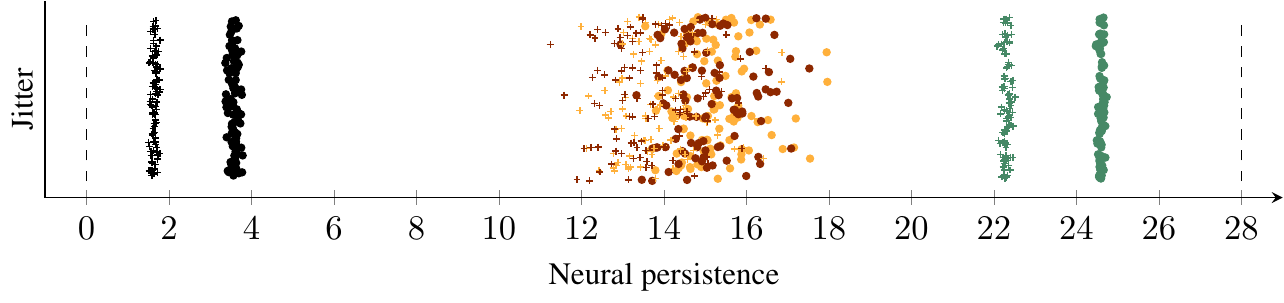}
    \else
      \input{figures/tikz/figure1.tex}
    \fi
  \caption{%
    Neural persistence values of trained perceptrons~(\textcolor{printable_1}{green}),
    diverging ones~(\textcolor{printable_2}{yellow}), random Gaussian
    matrices~(\textcolor{printable_3}{red}), and random uniform
    matrices~(\textcolor{printable_4}{black}).
    We performed 100 runs per category;
    dots indicate neural persistence while crosses indicate the
    predicted lower bound according to Theorem~\ref{thm:Tighter bounds}.
    The bounds according to Theorem~\ref{thm:Neural persistence bounds}
    are shown as dashed lines.
  }
  \label{fig:Neural persistence regimes}
\end{figure}

\paragraph{Complexity regimes in neural persistence}
%
As an application of the two theorems, we briefly take a look at how
neural persistence changes for different classes of simple neural
networks. To this end, we train a perceptron on the `MNIST' data set.
Since our measure uses the weight matrix of a perceptron, we can
compare its neural persistence with the neural persistence of random
weight matrices, drawn from different distributions. Moreover, we can
compare trained networks with respect to their initial parameters.
Figure~\ref{fig:Neural persistence regimes} depicts the neural
persistence values as well as the lower bounds according to
Theorem~\ref{thm:Tighter bounds} for different settings. We can see that
a network in which the optimizer diverges~(due to improperly selected
parameters) is similar to a random Gaussian matrix. Trained networks, on
the other hand, are clearly distinguished from all other networks.
Uniform matrices have a significantly lower neural persistence than
Gaussian ones. This is in line with the intuition that the latter type
of networks induces functional sparsity because few neurons have
large absolute weights.
For clarity, we refrain from showing the empirical upper bounds because
most weight distributions are highly right-tailed; the bound
will not be as tight as the lower bound.
These results are in line with a previous analysis~\citep{Sizemore17} of
small weighted networks, in which persistent homology is seen to
outperform traditional graph-theoretical complexity measures such as
the clustering coefficient~(see also Section~\ref{sec:Comparison graph theory}
in the appendix).
For deeper networks, additional experiments discuss the relation between
validation accuracy and neural persistence~(Section~\ref{sec:val}), the
impact of different data distributions, as well as the variability of
neural persistence for architectures of varying
depth~(Section~\ref{sec:datadistrib}).

\section{Experiments}\label{sec:Experiments}

This section demonstrates the utility and relevance of neural
persistence for fully connected deep neural networks.
We examine how commonly used regularization techniques~(batch
normalization and dropout) affect neural persistence of trained
networks.
Furthermore, we develop an early stopping criterion based on neural
persistence and we compare it to the traditional criterion based on
validation loss.
We used different architectures with \emph{ReLU} activation functions
across experiments. The brackets denote the number of units per
\emph{hidden} layer. In addition, the Adam optimizer with hyperparameters tuned via
cross-validation was used unless noted otherwise.
Please refer to Table~\ref{tab:best} in the appendix for further details
about the experiments.

\begin{figure}[bp]
  \centering
  \iffinal
    \includegraphics{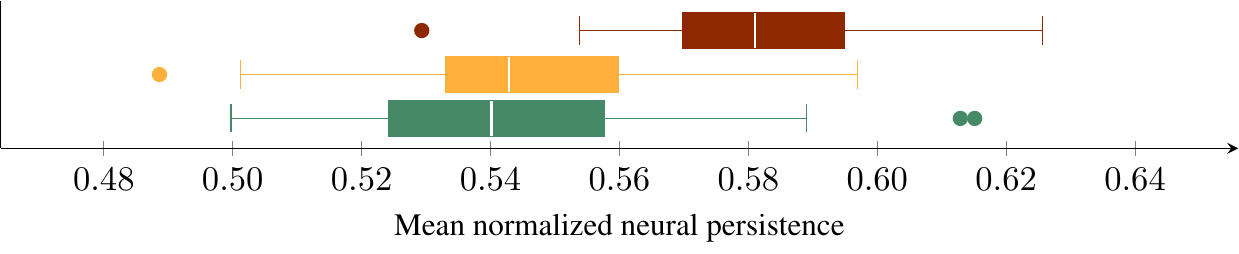}
  \else
    \input{figures/tikz/figure2.tex}
  \fi
  \caption{%
     Comparison of mean normalized neural persistence for trained
     networks without modifications~(\textcolor{printable_1}{green}),
     with batch normalization~(\textcolor{printable_2}{yellow}), and
     with 50\% of the neurons dropped out during
     training~(\textcolor{printable_3}{red}) for the `MNIST' data set~(50
     runs per setting).
  }
  \label{fig:comparison-regularizations}
\end{figure}

\subsection{Deep learning best practices in light of neural persistence}

We compare the mean normalized neural persistence~(see Definition~\ref{def:mean-normalized-np})
of a two-layer~(with an architecture of $[650,650]$) neural network to two models where batch normalization~\citep{ioffe2015batch} or dropout~\citep{srivastava2014dropout} are applied.
Figure~\ref{fig:comparison-regularizations} shows that the networks
designed according to best practices yield higher normalized neural
persistence values on the `MNIST' data set in comparison to an unmodified network.
The effect of dropout on the mean normalized neural persistence is more
pronounced and this trend is directly analogous to the observed accuracy
on the test set.
These results are consistent with expectations if we consider dropout to
be similar to ensemble learning~\citep{hara2016analysis}. As individual
parts of the network are trained independently, a higher degree of
per-layer redundancy is expected, resulting in a different structural
complexity.
Overall, these results indicate that for a fixed architecture approaches targeted at increasing
the neural persistence during the training process may be of particular
interest.

\subsection{Early stopping based on neural persistence}
\label{sec:Early stopping}

Neural persistence can be used as an \emph{early stopping}
criterion that does not require a validation data set to
prevent overfitting: if the mean normalized neural persistence does not
increase by more than $\Delta_{\min}$ during a certain number of epochs
$g$, the training process is stopped. This procedure is called
`patience' and Algorithm~\ref{alg:Early stopping} describes it in
detail.
A similar variant of this algorithm, using validation loss instead of
persistence, is the state-of-the-art for early stopping in
training~\citep{Bengio12, Keras15}.
To evaluate the efficacy of our measure, we compare it against
validation loss in an extensive set of scenarios. More precisely, for
a training process with at most $G$ epochs, we define a $G \times G$
parameter grid consisting of the `patience' parameter~$g$ and
a burn-in rate~$b$~(both measured in epochs).
$b$ defines the number of epochs after which an early stopping criterion
starts monitoring, thereby preventing underfitting.
Subsequently, we set $\Delta_{\min} = 0$ for all measures to remain
comparable and scale-invariant, as non-zero values could implicitly
favour one of them due to scaling.
For each data set, we perform 100 training runs of the same
architecture, monitoring validation loss and mean normalized neural
persistence every quarter epoch.
The early stopping behaviour of both measures is simulated for each
combination of $b$ and $g$ and their performance over all runs is
summarized in terms of median test accuracy and median stopping epoch;
if a criterion is not triggered for one run, we report the test accuracy
at the end of the training and the number of training epochs.
This results in a scatterplot, where each point~(corresponding to
a single parameter combination) shows the difference in epochs and
the absolute difference in test accuracy~(measured in percent).
The quadrants permit an intuitive explanation: $Q_2$, for example,
contains all configurations for which our measure stops \emph{earlier},
while achieving a \emph{higher} accuracy.
Since $b$ and $g$ are typically chosen to be small in an early stopping
scenario, we use grey points to indicate uncommon configurations for
which $b$ or $g$ is larger than half of the total number of epochs.
Furthermore, to summarize the performance of our measure, we calculate
the barycentre of all configurations~(green square).

\begin{algorithm}[bp]
  \caption{Early stopping based on mean normalized neural persistence}
  \label{alg:Early stopping}
  \algorithmicrequire{} Weighted neural network~$\mathcal{N}$, patience $g$, $\Delta_{\min}$
  \begin{algorithmic}[1]
    \State $P \gets 0$, $G \gets 0$               \Comment{Initialize highest observed value and patience counter}
    \Procedure{EarlyStopping}{$\mathcal{N}$, $g$, $\Delta_{\min}$}           \Comment{Callback that monitors training at every epoch}
      \State $P' \gets \meanneuralpersistence(\mathcal{N})$ \label{earlystop:neuralpers}
      \If{$P' > P + \Delta_{\min}$}               \Comment{Update mean normalized neural persistence and reset counter}
        \State $P \gets P'$, $G \gets 0$ 
      \Else                                       \Comment{Update patience counter}
        \State $G \gets G+1$
      \EndIf
      \If{$G \geq g$}                             \Comment{Patience criterion has been triggered}
        \State \textbf{return} $P$                \Comment{Stop training and return highest observed value}
      \EndIf{}
    \EndProcedure
  \end{algorithmic}
\end{algorithm}

\begin{figure}[tbp]
  \centering
  \raisebox{0.80\height}{
    \hspace*{-1.75cm}
    \begin{minipage}{0.45\textwidth}
      \centering
      \subcaptionbox{Fashion-MNIST\label{sfig:FMNIST}}{%
        \hspace*{-0.75cm}
        \iffinal
          \includegraphics{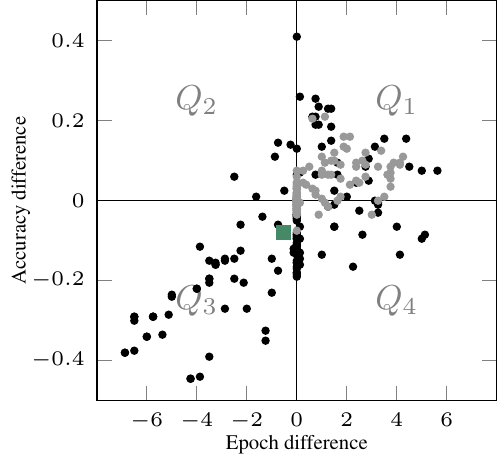}
        \else
          \input{figures/tikz/figure3.tex}
        \fi
      }
      \subcaptionbox{Summary}{%
        \tiny
        \begin{tabular}{lll}
            Data set       & Barycentre      & Final test accuracy\\
            \midrule
            Fashion-MNIST  & $(-0.53,-0.08)$ & $86.72 \pm 0.43$\\ 
            MNIST          & $(+0.17,-0.06)$ & $96.16 \pm 0.24$\\
            CIFAR-10       & $(-1.33,-1.13)$ & $52.19 \pm 3.40$\\
            IMDB           & $(-1.68,+0.07)$ & $87.35 \pm 0.03$\\
        \end{tabular}
      }
  \end{minipage}
  }
  \hspace*{-0.5cm}
  \raisebox{0.80\height}{
    \begin{minipage}{0.45\textwidth}
      \subcaptionbox{Accuracy and epoch differences\label{sfig:FMNIST deltas}}{
        \includegraphics[height=3cm]{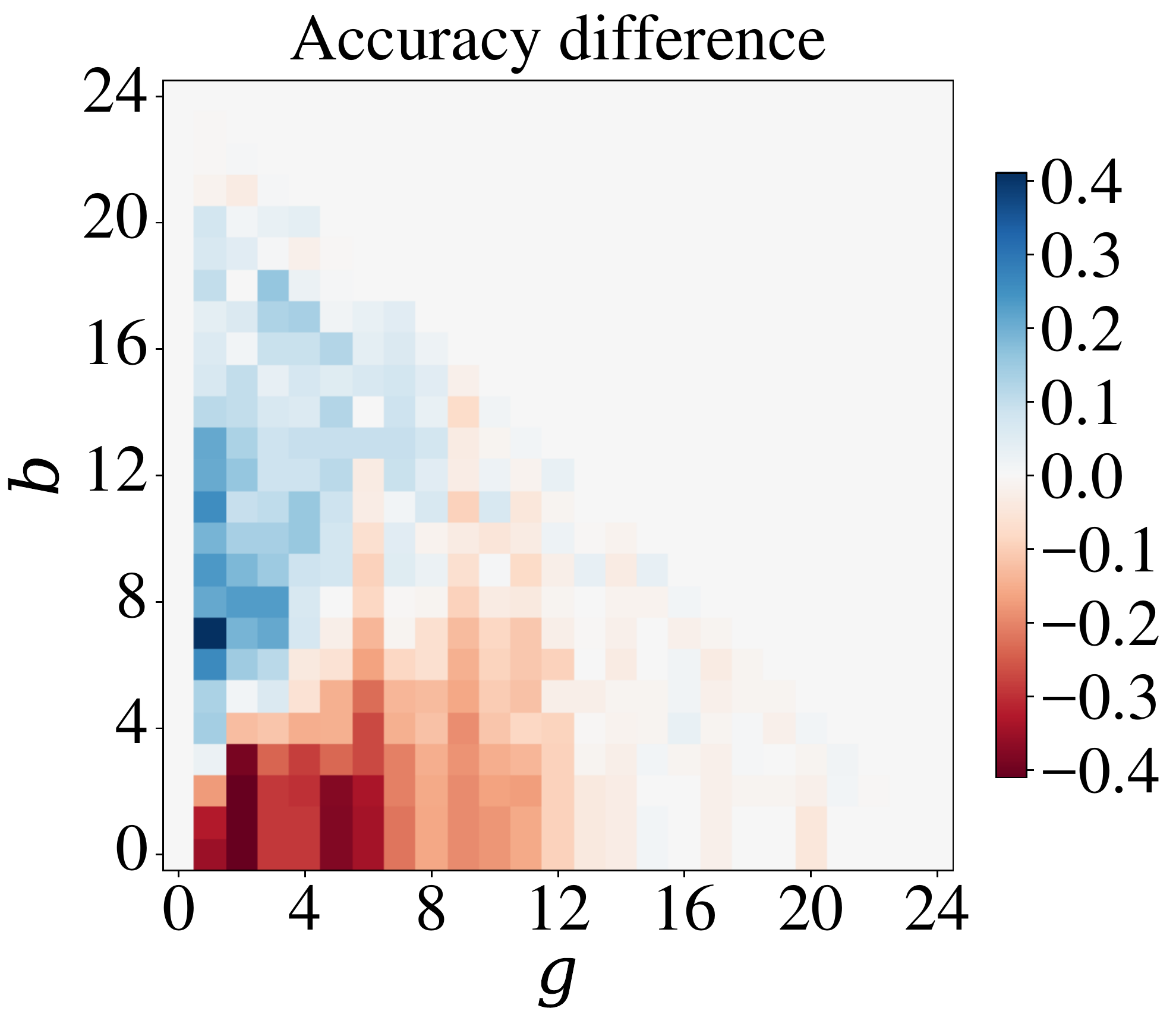}
        \includegraphics[height=3cm]{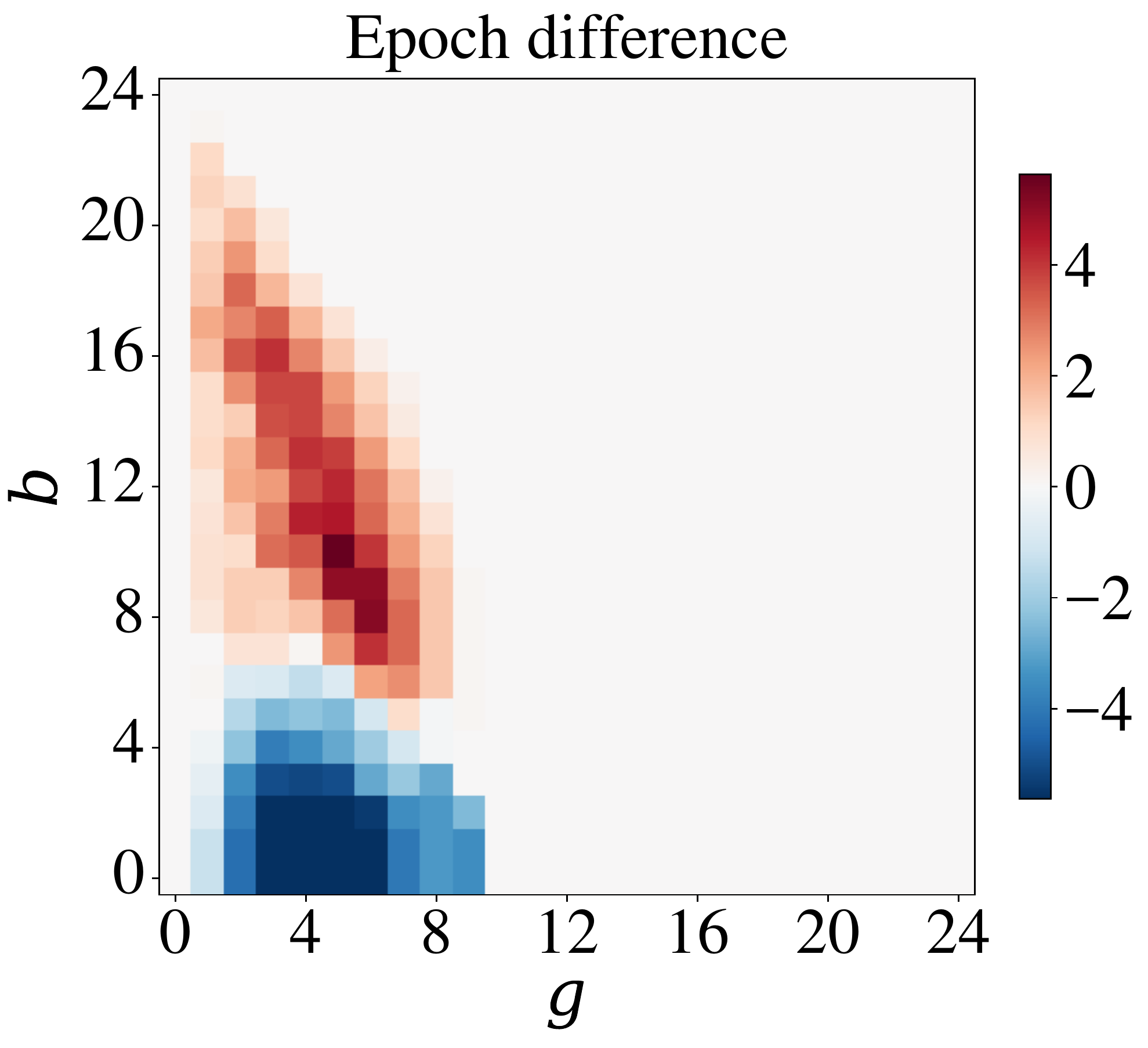}
      }
      \subcaptionbox{Number of triggers\label{sfig:FMNIST triggers}}{
        \includegraphics[height=3cm]{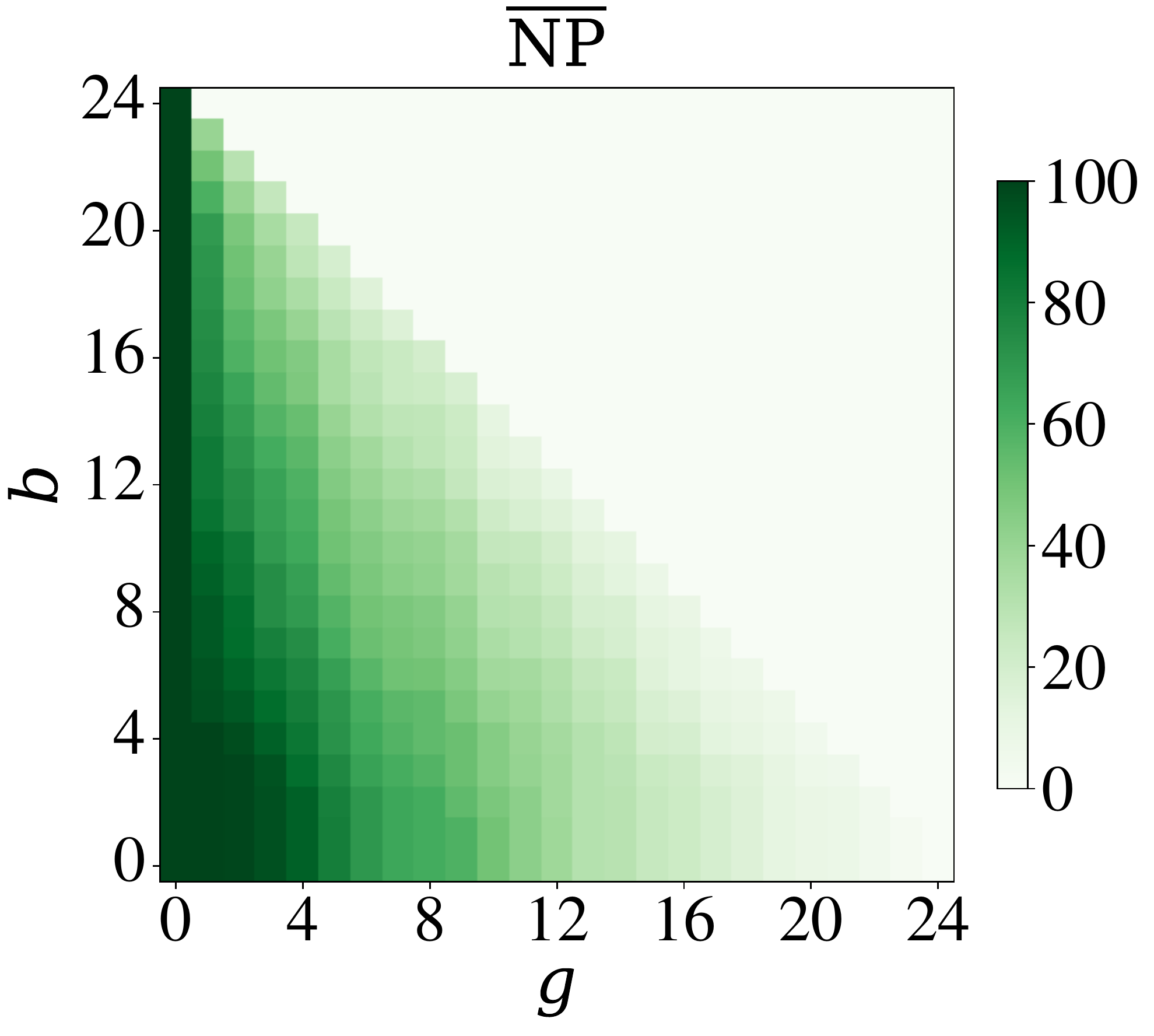}
        \hspace{0.1cm}
        \includegraphics[height=3cm]{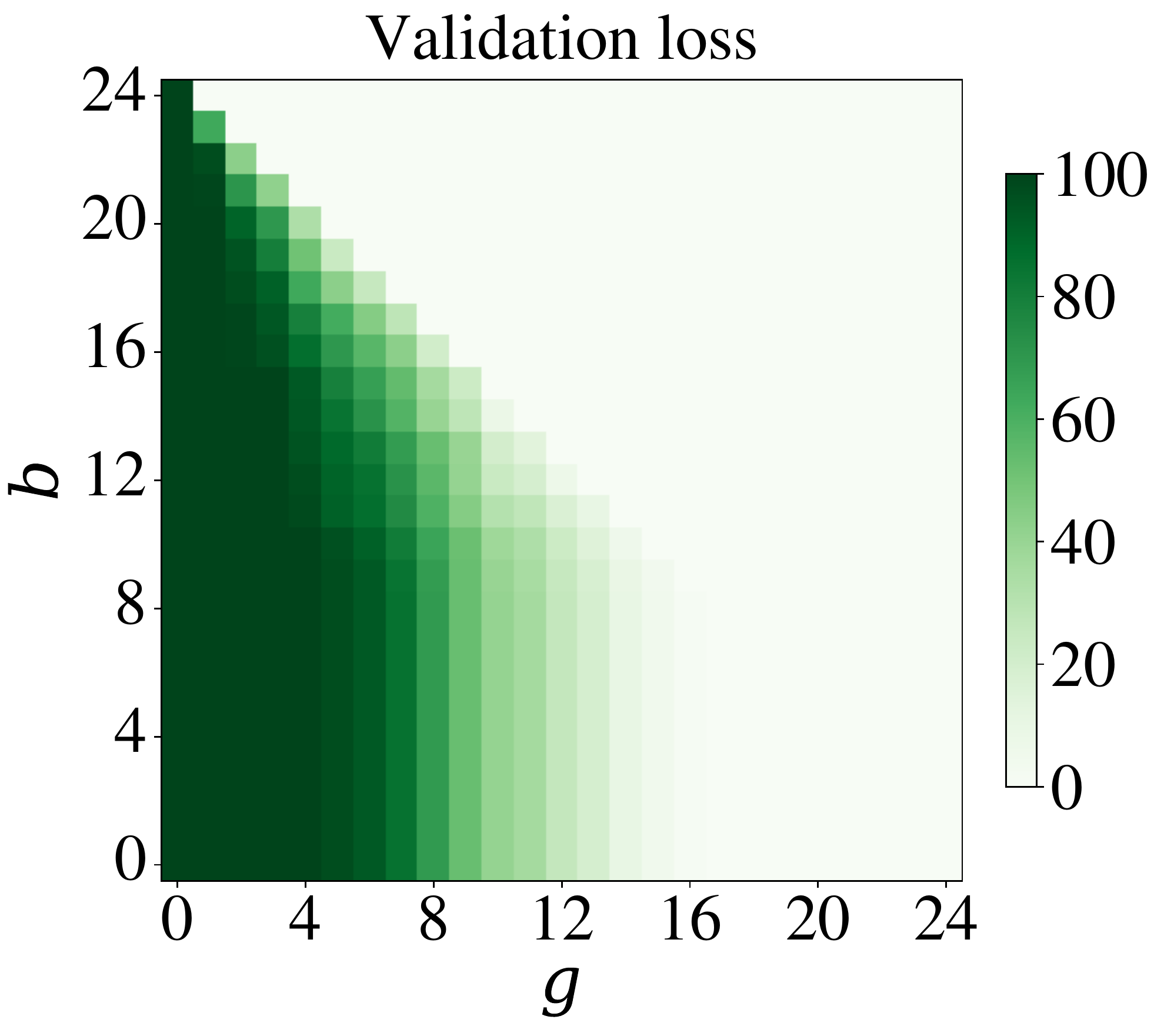}
      }
    \end{minipage}
  }
  \caption{%
    The visualizations depict the differences in accuracy and
    epoch for all comparison scenarios of mean normalized neural
    persistence versus validation loss, while the table
    summarizes the results on other data sets. Final test
    accuracies are shown irrespectively of early stopping to put the
    accuracy differences into context.
  }
  \label{fig:Comparison early stopping}
  \vskip -0.5cm
\end{figure}

Figure~\ref{sfig:FMNIST} depicts the comparison with validation loss for
the `Fashion-MNIST'~\citep{xiao2017fashion} data set; please refer to
Section~\ref{sec:Early stopping plus} in the appendix for more data
sets.
Here, we observe that most common configurations are in $Q_2$ or in
$Q_3$, i.e\ our criterion stops earlier.
The barycentre is at $(-0.53, -0.08)$, showing that out of 625
configurations, on average we stop half an epoch earlier than validation
loss, while losing virtually no accuracy~($0.08\%$).
Figure~\ref{sfig:FMNIST deltas} depicts detailed differences in accuracy
and epoch for our measure when compared to validation loss; each cell in
a heatmap corresponds to a single parameter configuration of $b$ and
$g$.
In the heatmap of accuracy differences, blue, white, and red represent
parameter combinations for which we obtain \emph{higher, equal, or
lower} accuracy, respectively, than with validation loss for the same
parameters.
Similarly, in the heatmap of epoch differences, green represents
parameter combinations for which we stop \emph{earlier} than validation
loss.
For $b \leq 8$, we stop earlier~($0.62$ epochs on average), while losing
only $0.06\%$ accuracy.
Finally, Figure~\ref{sfig:FMNIST triggers} shows how often each measure
is triggered.
Ideally, each measure should consist of a dark green triangle, as this
would indicate that \emph{each} configuration stops all the time. For
this data set, we observe that our method stops for more parameter
combinations than validation loss, but not as frequently for all of
them.
To ensure comparability across scenarios, we did not use the validation
data as additional training data when stopping with neural persistence;
we refer to Section~\ref{sec:extreme} for additional experiments in data
scarcity scenarios. We observe that our method stops earlier when
overfitting can occur, and it stops later when longer training is
beneficial. 

\section{Discussion}\label{sec:Discussion}

In this work, we presented \emph{neural persistence}, a novel topological
measure of the structural complexity of deep neural networks.
We showed that this measure captures topological information that
pertains to deep learning performance.
Being rooted in a rich body of research, our measure is theoretically
well-defined and, in contrast to previous work, generally applicable
as well as computationally efficient.
We showed that our measure correctly identifies networks that employ
best practices such as dropout and batch normalization.
Moreover, we developed an early stopping criterion that exhibits
competitive performance while not relying on a separate validation data
set. Thus, by saving valuable data for training, we managed to boost
accuracy, which can be crucial for enabling deep learning in regimes of
smaller sample sizes.
Following Theorem~\ref{thm:Tighter bounds}, we also experimented with
using the $p$-norm of \emph{all} weights of the neural network as
a proxy for neural persistence. However, this did not yield an early stopping
measure because it was never triggered, thereby suggesting that neural
persistence captures salient information that would otherwise be hidden
among all the weights of a network.
We extended our framework to convolutional neural networks~(see
Section~\ref{sec:cnn}) by deriving a closed-form approximation, and
observed that an early stopping criterion based on neural persistence
for convolutional layers will require additional work. 
Furthermore, we conjecture that assessing dissimilarities of networks by
means of persistence diagrams~(making use of higher-dimensional
topological features), for example, will lead to further insights regarding
their generalization and learning abilities. Another interesting avenue
for future research would concern the analysis of the `function space'
learned by a neural network.
On a more general level, \emph{neural persistence} demonstrates the great
potential of topological data analysis in machine learning.

\bibliography{main}
\bibliographystyle{iclr2019_conference}

\clearpage
\appendix

\counterwithin{figure}{section}
\counterwithin{table} {section}

\section{Appendix}\label{sec:Appendix}

\subsection{Comparison with graph-theoretical measures}
\label{sec:Comparison graph theory}
Traditional complexity/structural measures from graph theory, such as the
clustering coefficient, the average shortest path length, and
global/local efficiency are already known to be insufficiently accurate
to characterize different models of complex random networks~\cite{Sizemore17}.
Our experiments indicate that this holds true for~(deep) neural
networks, too. As a brief example, we trained a perceptron on the MNIST data
set with batch stochastic gradient descent~($\eta = 0.5$), achieving
a test accuracy of $\approx 0.91$. Moreover, we intentionally
`sabotaged' the training by setting $\eta = \num{1e-5}$ such that SGD is
unable to converge properly. This leads to networks with accuracies
ranging from $0.38$--$0.65$.
A complexity measure should be capable of distinguishing both classes of
networks.
However, as Figure~\ref{fig:Motivation}~(top) shows, this is \emph{not}
the case for the clustering coefficient. Neural persistence~(bottom), on
the other hand, results in two regimes that can clearly be distinguished,
with the trained networks having a significantly smaller variance.

\begin{figure}[tbp]
  \centering
  \iffinal
    \includegraphics{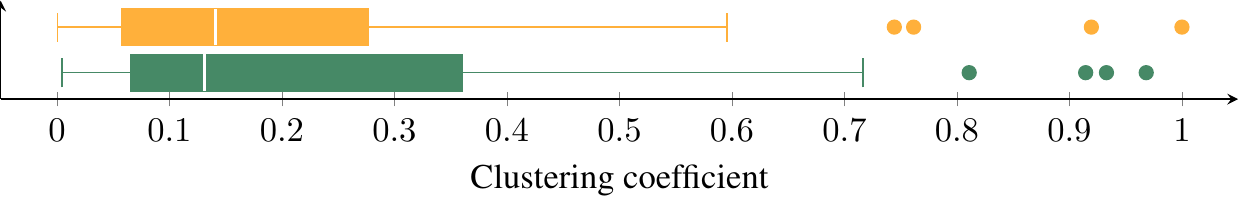}
    \includegraphics{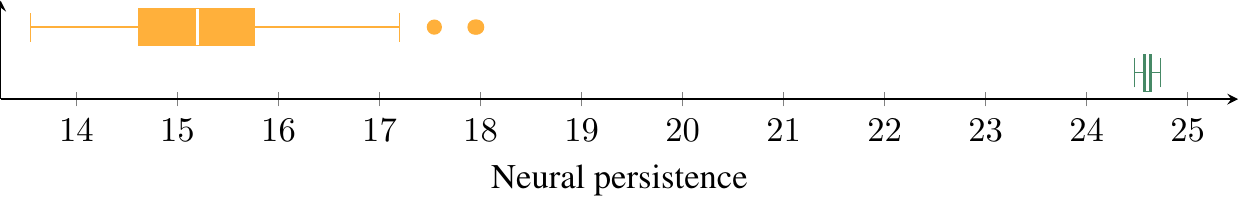}
  \else
    \input{figures/tikz/figure4.tex}
    \input{figures/tikz/figure5.tex}
  \fi
  \caption{%
    Traditional graph measures~(top), such as the clustering
    coefficient, fail to detect differences in the complexity
    of neural networks.
    Our novel \emph{neural persistence} measure~(bottom), by contrast,
    shows that trained networks with $\eta = 0.5$~(\textcolor{printable_1}{green}),
    which have an accuracy of $\approx 0.91$, obey a different
    distribution than networks trained with $\eta
    = \num{1e-0.5}$~(\textcolor{printable_2}{yellow}), which have accuracies ranging
    from $0.38$--$0.65$.
  }
  \label{fig:Motivation}
\end{figure}

\subsection{Proof of Theorem~\ref{thm:Tighter bounds}}\label{app:Proofs}
\begin{proof}
  We may consider the filtration from Section~\ref{sec:Neural persistence}
  to be a subset selection problem with constraints, where we select $n$
  out of $m$ weights.
  The neural persistence $\neuralpersistence(G_k)$ of a layer thus only
  depends on the \emph{selected} weights that appear as tuples of the form
  $(1,w_i)$ in $\diagram_k$.
  Letting $\widetilde{\mathbf{w}}$ denote the vector of selected weights
  arising from the persistence diagram calculation, we can rewrite neural
  persistence as
  $\neuralpersistence(G_k) = \norm{\mathds{1} - \widetilde{\mathbf{w}}}_p$.
  Furthermore, $\widetilde{\mathbf{w}}$ satisfies
  $\norm{\mathbf{w}_{\min}}_p \leq \norm{\widetilde{\mathbf{w}}}_p \leq \norm{\mathbf{w}_{\max}}_p$.
  Since all transformed weights are non-negative in our filtration, it
  follows that~(note the reversal of the two terms)
  \begin{equation}
  \norm{\mathds{1} - \mathbf{w}_{\max}}_p \leq \neuralpersistence(G_k) \leq \norm{\mathds{1} - \mathbf{w}_{\min}}_p,
  \end{equation}
  and the claim follows.
\end{proof}

\subsection{Additional visualizations and analyses for early stopping}
\label{sec:Early stopping plus}

\begin{figure}[p]
  \centering
  \subcaptionbox{Accuracy and epoch differences}{
    \includegraphics[height=4.5cm]{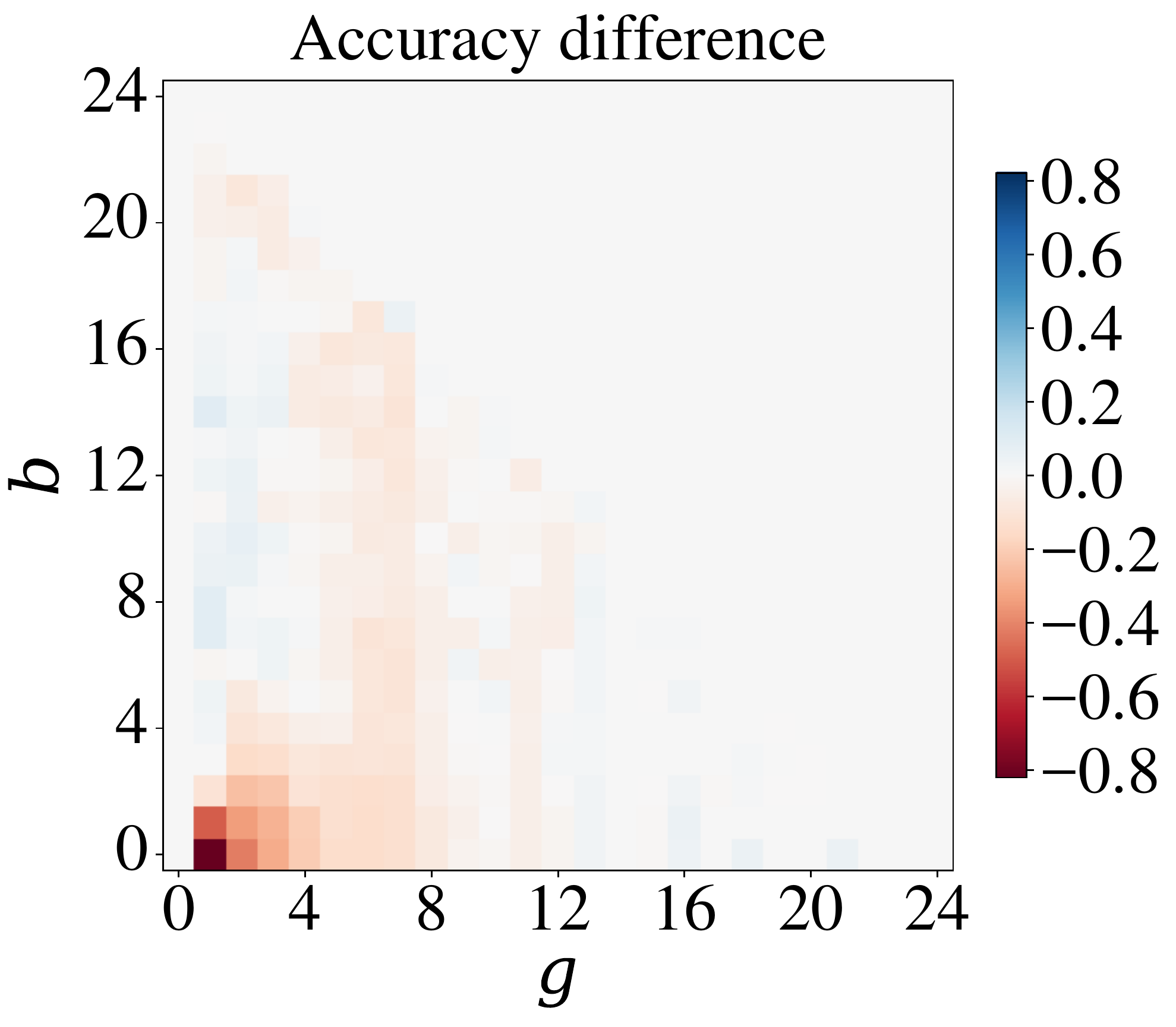}
    \includegraphics[height=4.5cm]{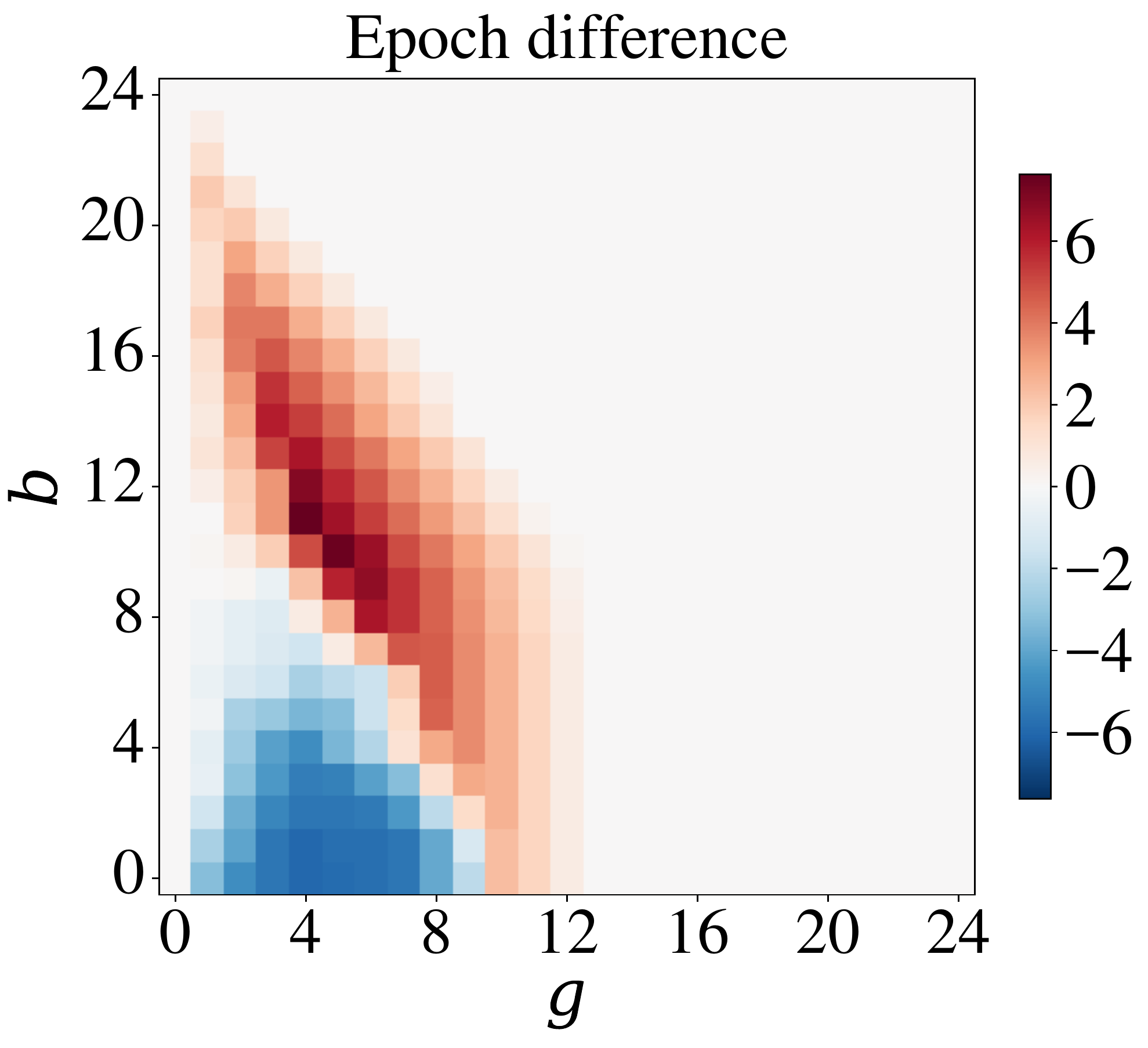}
  }
  \subcaptionbox{Number of triggers}{
    \includegraphics[height=4.5cm]{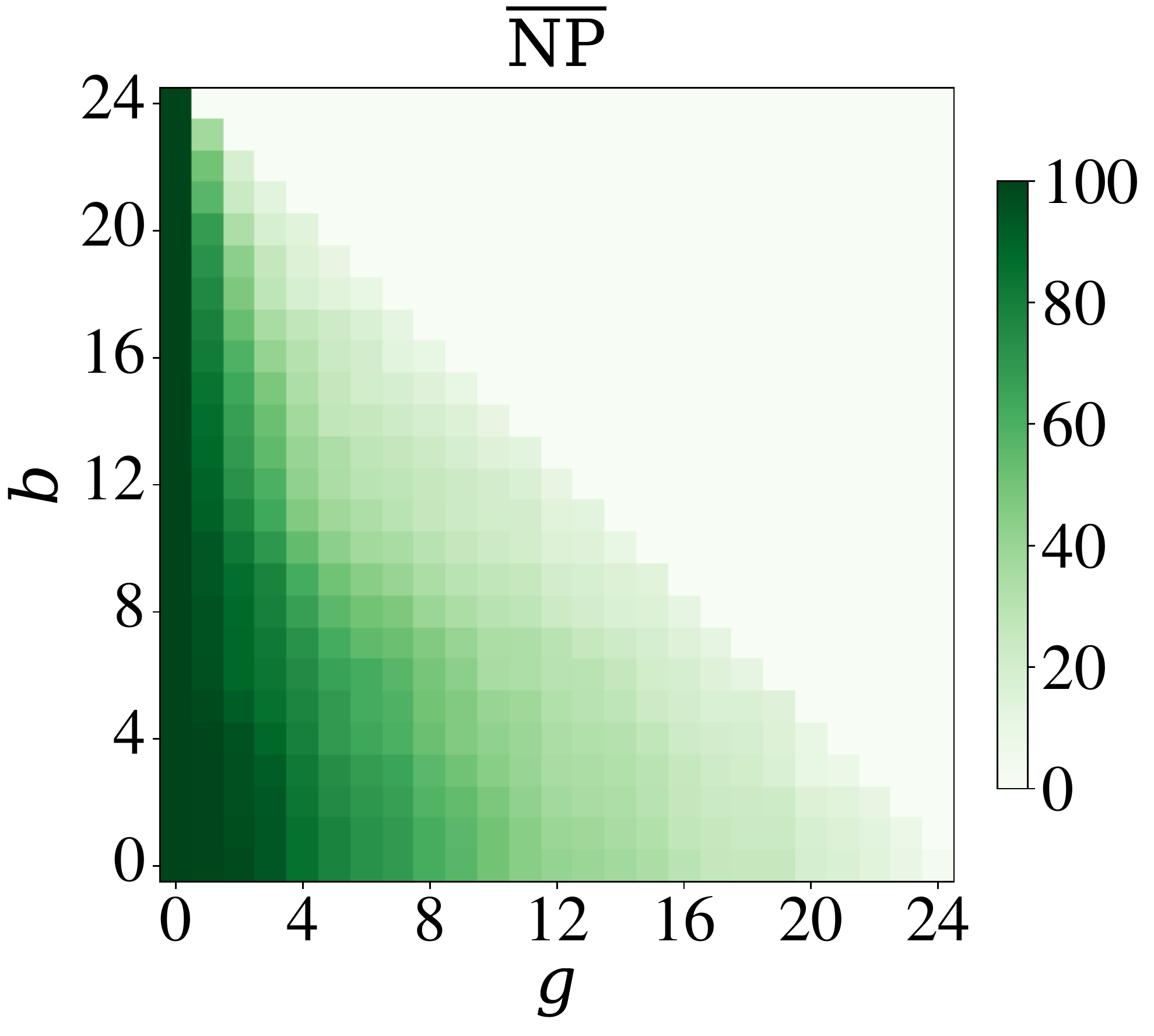}
    \includegraphics[height=4.5cm]{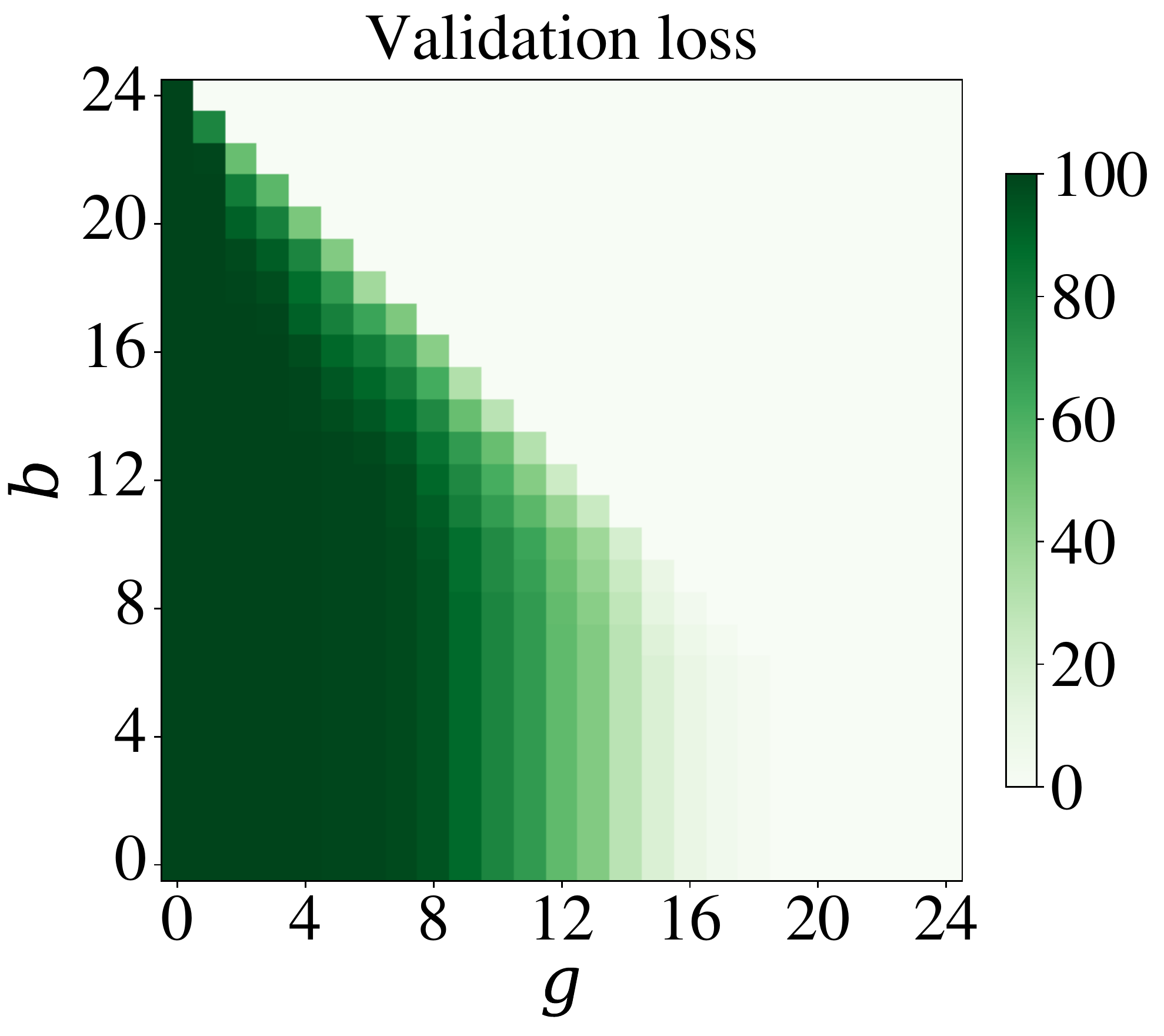}
  }
  \caption{%
    Additional visualizations for the `MNIST' data set.
  }
  \label{fig:MNIST additional}
\end{figure}

Due to space constraints and the large number of configurations that we
investigated for our early stopping experiments, this section  contains
additional plots that follow the same schematic: the top row shows  the
differences in accuracy and epoch for our measure when compared to  the
commonly-used validation loss. Each cell in the heatmap corresponds  to
a single configuration of $b$ and $g$.
In the heatmap of accuracy differences, blue represents parameter
combinations for which we obtain a \emph{higher} accuracy than
validation loss for the same parameters; white indicates combinations
for which we obtain the same accuracy, while red highlights combinations
in which our accuracy decreases.
Similarly, in the heatmap of epoch differences, green represents
parameter combinations for which we stop \emph{earlier} than validation
loss for the same parameter.
The scatterplots in Section~\ref{sec:Early stopping} show an `unrolled'
version of this heat map, making it possible to count how many parameter
combinations result in early stops while also increasing accuracy, for
example. The heatmaps, by contrast, make it possible to compare the
behaviour of the two measures with respect to each parameter combination.
Finally, the bottom row of every plot shows how many times each measure
was triggered for every parameter combination. We consider a measure to
be triggered if its stopping condition is satisfied prior to the last
training epoch. Due to the way the parameter grid is set up, no configuration
above the diagonal can stop, because $b + g$ would be larger than the
total number of training epochs. This permits us to compare the
`slopes' of cells for each measure. Ideally, each measure should
consist of a dark green triangle, as this would indicate that
\emph{parameter} configuration stops all the time.

\paragraph{MNIST} Please refer to Figures~\ref{fig:MNIST additional} and \ref{fig:MNIST scatterplot}.
The colours in the difference matrix of the top row are slightly skewed
because in a certain configuration, our measure loses $0.8\%$ of
accuracy when stopping. However, there are many other configurations in
which virtually no accuracy is lost and in which we are able to stop
more than four epochs earlier. The heatmaps in the bottom row again
indicate that neural persistence is capable of stopping for more
parameter combinations in general. We do not trigger as often for some
of them, though.

\begin{figure}
  \centering
  \iffinal
    \includegraphics{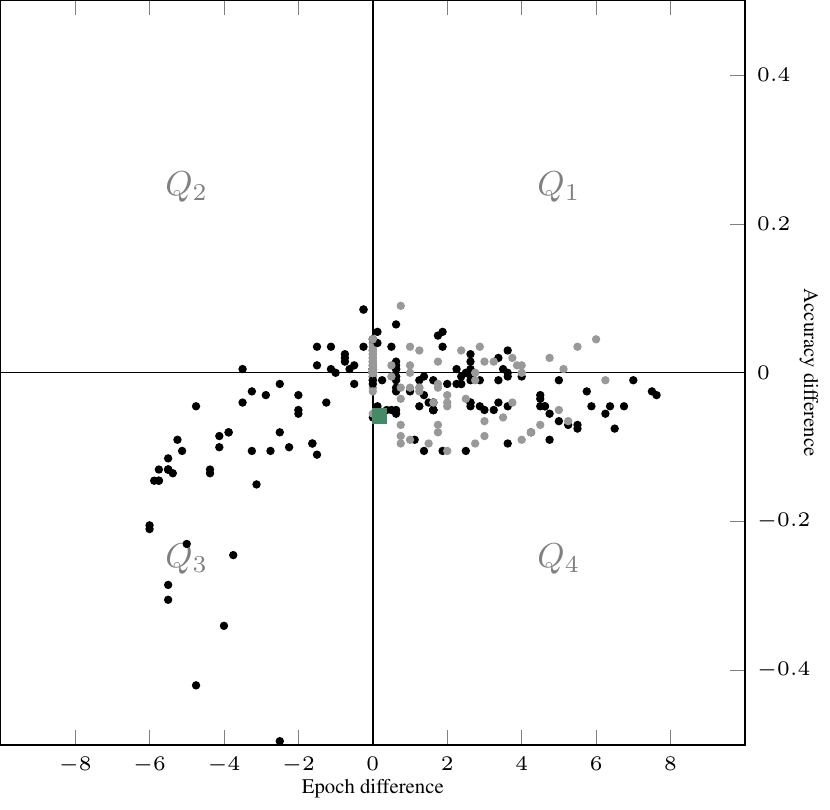}
  \else
    \input{figures/tikz/figure6.tex}
  \fi
  \caption{%
    Scatterplot of epoch and accuracy differences for `MNIST'.
  }
  \label{fig:MNIST scatterplot}
\end{figure}

\begin{figure}[p]
  \centering
  \subcaptionbox{Accuracy and epoch differences}{
    \includegraphics[height=4.5cm]{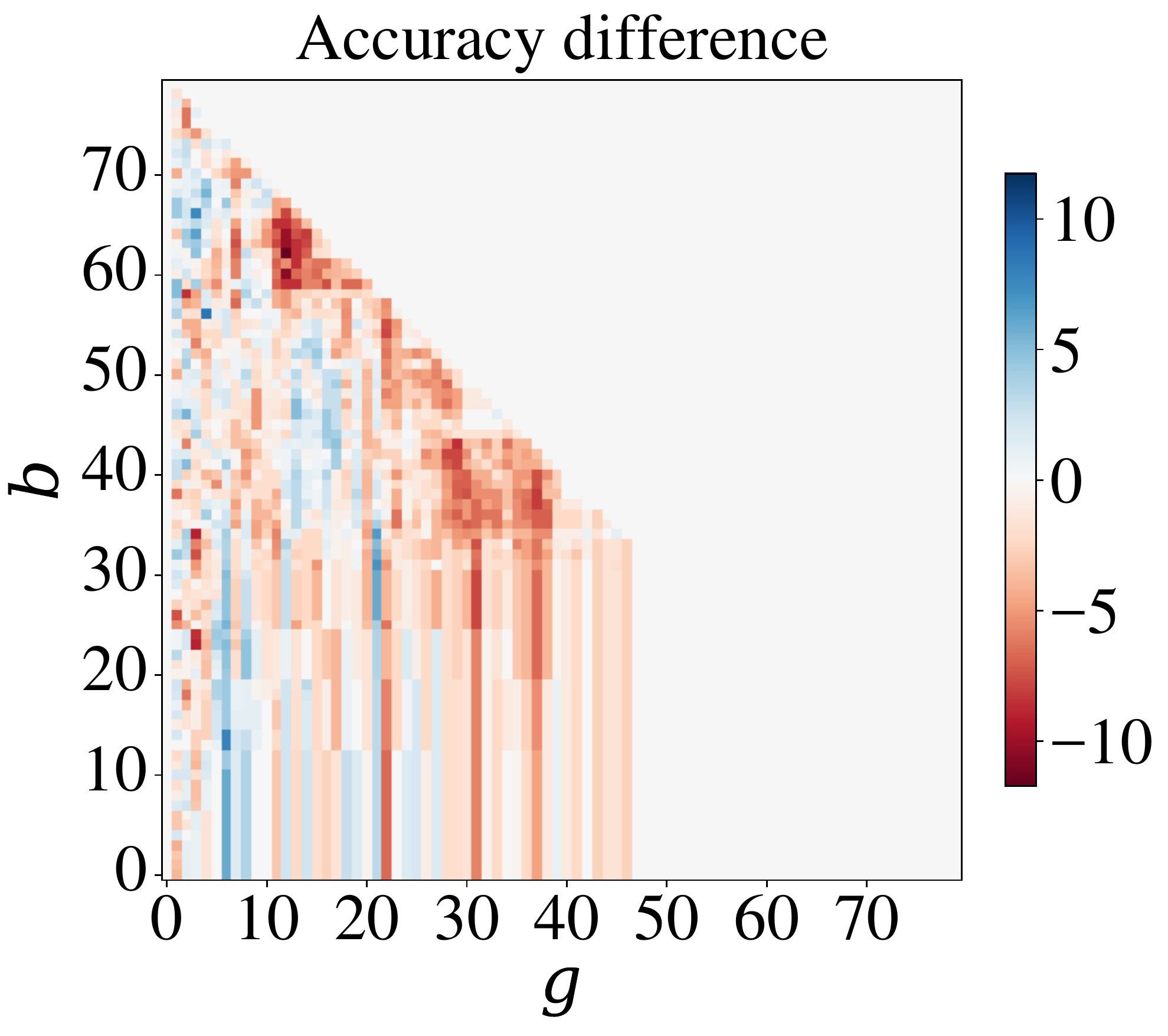}
    \includegraphics[height=4.5cm]{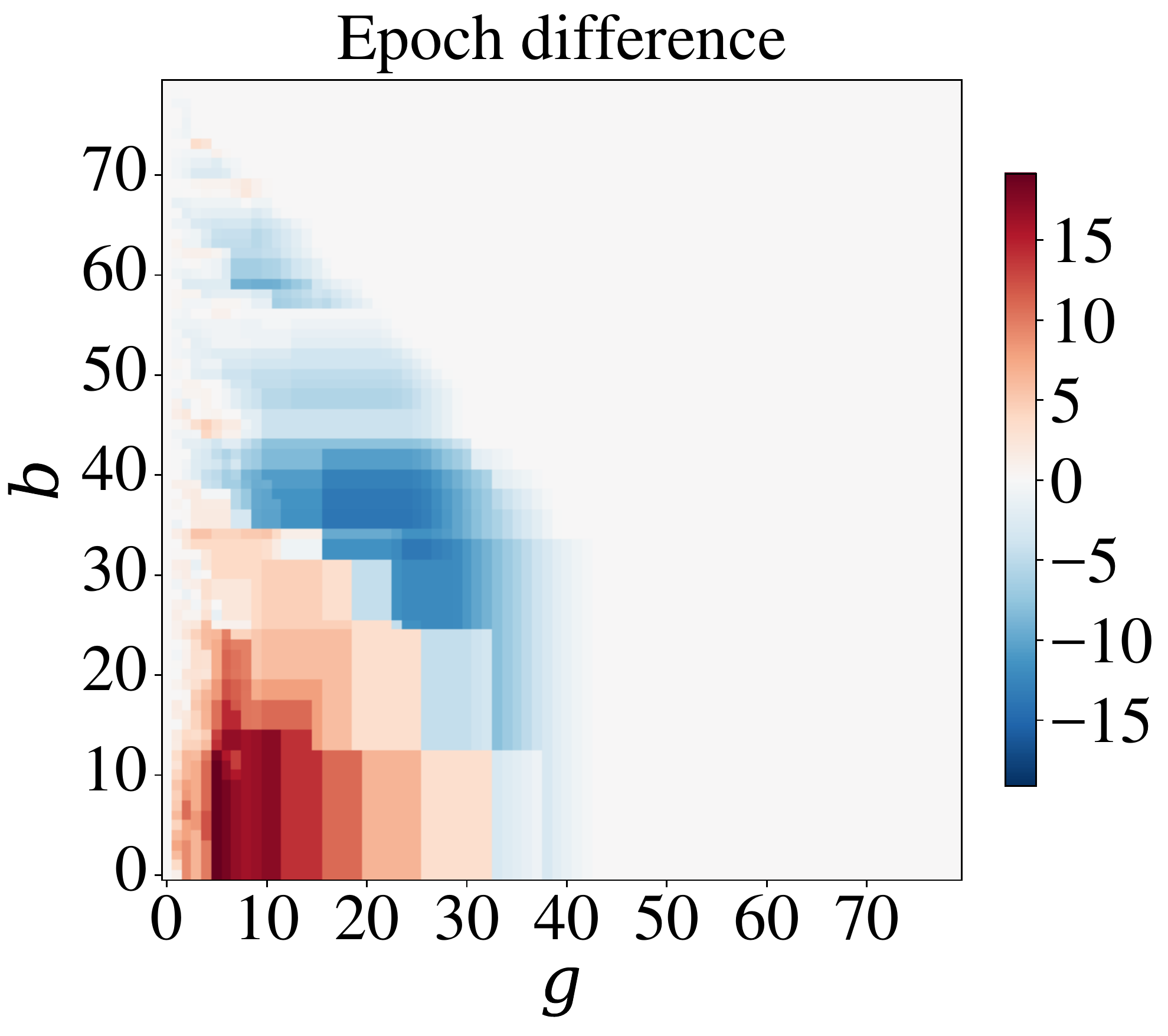}
  }
  \subcaptionbox{Number of triggers}{%
    \includegraphics[height=4.5cm]{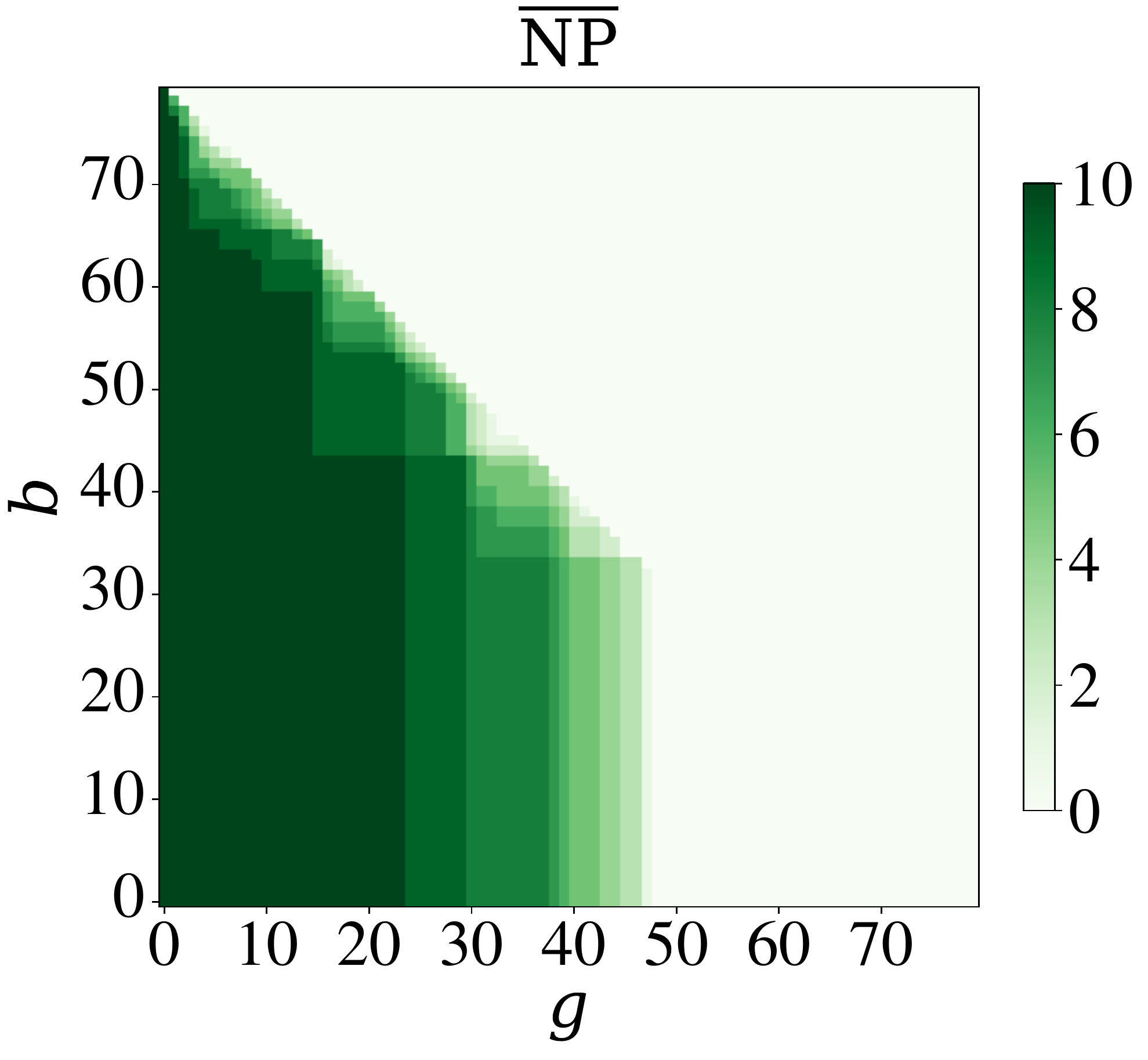}
    \includegraphics[height=4.5cm]{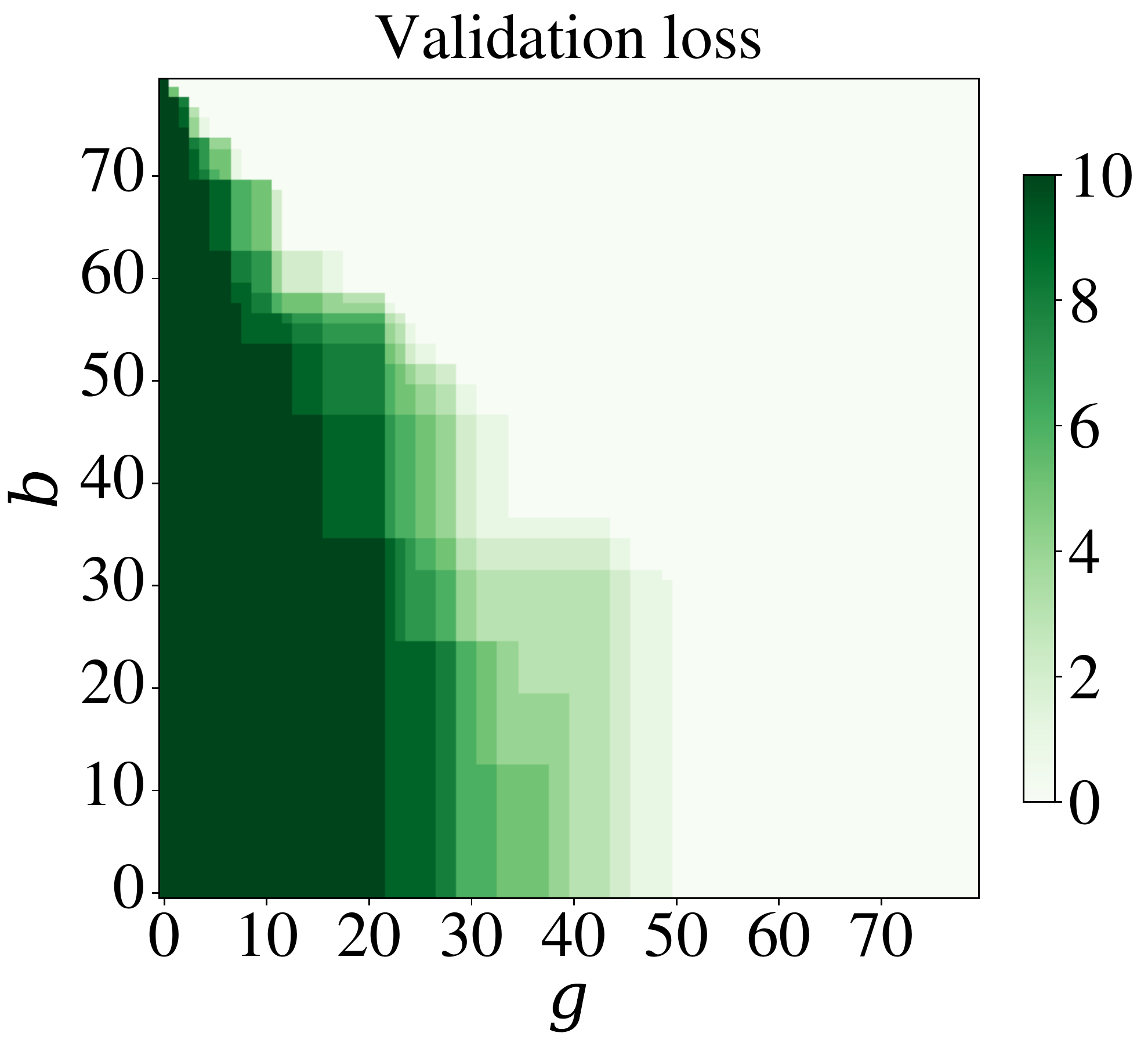}
  }
  \caption{%
    Additional visualizations for the `CIFAR-10' data set.
  }
  \label{fig:CIFAR additional}
\end{figure}

\begin{figure}[p]
  \centering
  \iffinal
    \includegraphics{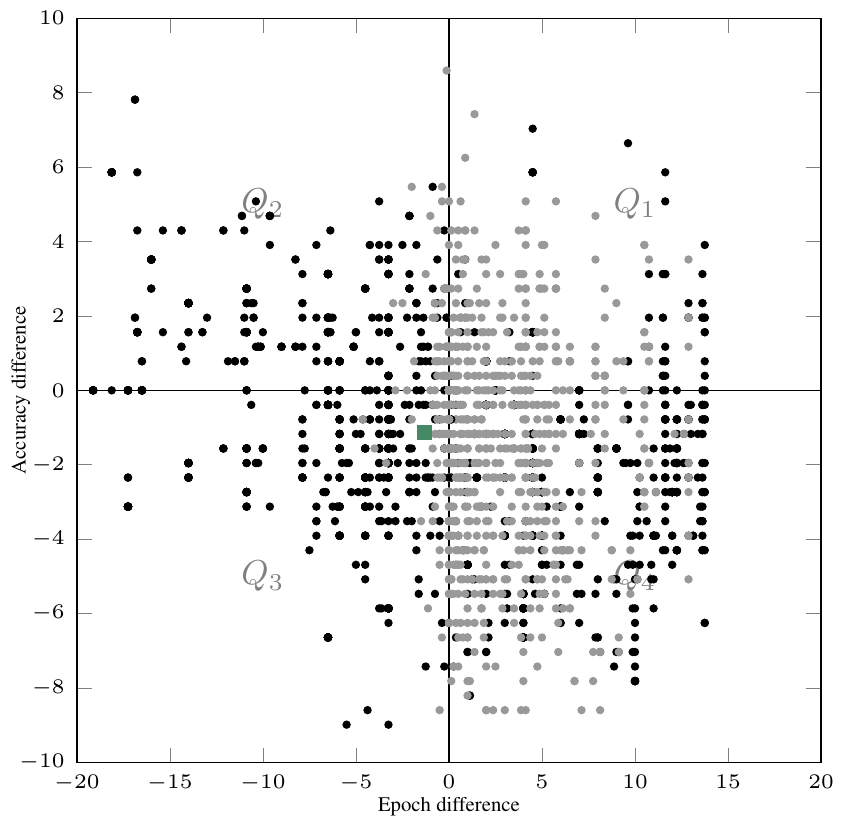}
  \else
    \input{figures/tikz/figure7.tex}
  \fi
  \caption{%
    Scatterplot of epoch and accuracy differences for `CIFAR-10'.
  }
  \label{fig:CIFAR scatterplot}
\end{figure}

\paragraph{CIFAR-10} Please refer to Figure~\ref{fig:CIFAR additional}.
In general, we observe that this data set is more sensitive with respect
to the parameters for early stopping. While there are several
configurations in which neural persistence stops with an increase of
almost $10\%$ in accuracy, there are also scenarios in which we cannot
stop training earlier, or have to train longer~(up to $15$ epochs out of
$80$ epochs in total).
The second row of plots shows our measure triggers reliably for more
configurations than validation loss. Overall, the scatterplot of all
scenarios~(Figure~\ref{fig:CIFAR scatterplot}) shows that most practical
configurations are again located in $Q_2$ and $Q_3$. While we may thus
find certain configurations in which we reliably outperform validation
loss as an early stopping criterion, we also want to point out that
our measures behaves correctly for many practical configurations. Points
in $Q_1$, where we train \emph{longer} and achieve a \emph{higher}
accuracy, are characterized by a high patience $g$ of approximately $40$
epochs and a low burn-in rate $b$, or vice versa. This is caused by the
training for CIFAR-10, which does not reliably converge for FCNs.
Figure~\ref{fig:Curves CIFAR vs. FMNIST} demonstrates this by showing
loss curves and the mean normalized neural persistence curves of five
runs over training~(loss curves have been averaged over all runs;
standard deviations are shown in grey; we show the first half of the
training to highlight the behaviour for practical early stopping
conditions).
For `Fashion-MNIST', we observe that $\meanneuralpersistence$
exhibits clear change points during the training process, which
can be exploited for early stopping. For `CIFAR-10', we observe
a rather incremental growth for some runs~(with no clearly-defined
maximum), making it harder to derive a generic early stopping criterion that
does not depend on fine-tuned parameters. Hence, we hypothesize that
neural persistence cannot be used reliably in scenarios where the
architecture is incapable of learning the data set.
In the future, we plan to experiment with deliberately selected `bad'
and `good' architectures in order to evaluate to what extent our
topological measure is capable of assessing their suitability for
training, but this is beyond the scope of this paper.

\paragraph{IMDB} Please refer to Figure~\ref{fig:IMDB additional}.
For this data set, we observe that most parameter configurations result
in \emph{earlier} stopping~(up to two epochs earlier than validation
loss), with accuracy increases of up to $0.10\%$. This is also shown in
the scatterplot~\ref{fig:IMDB scatterplot}. Only a single configuration,
viz.\ $g = 1$ and $b = 0$, results in a severe loss of accuracy; we
removed it from the scatterplot for reasons of clarity, as its accuracy
difference of $-21\%$ would skew the display of the remaining
configurations too much~(this is also why the legends do not include
this outlier).

\begin{figure}[tbp]
  \subcaptionbox{CIFAR-10}{%
    \includegraphics[width=\textwidth]{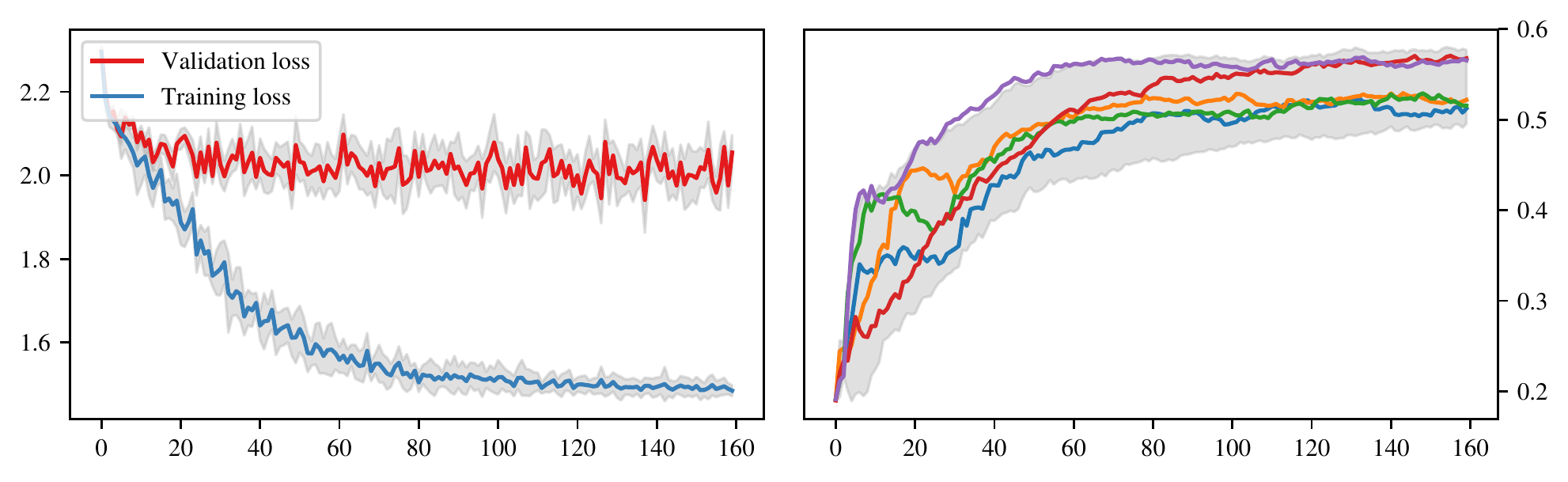}
  }
  \subcaptionbox{Fashion-MNIST}{%
    \includegraphics[width=\textwidth]{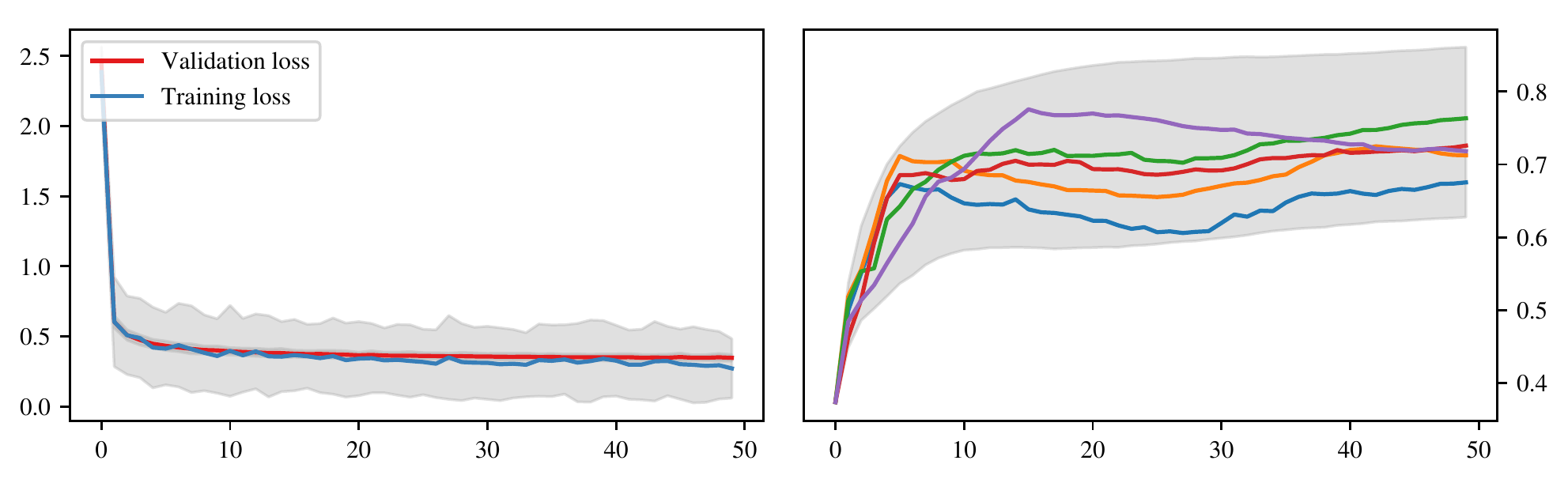}
  }
  \caption{
    A comparison of mean normalized neural persistence curves that we
    obtain during the training of `CIFAR-10' and `Fashion-MNIST'.
  }
  \label{fig:Curves CIFAR vs. FMNIST}
\end{figure}

\begin{figure}[p]
  \centering
  \subcaptionbox{Accuracy and epoch differences}{
    \includegraphics[height=4.5cm]{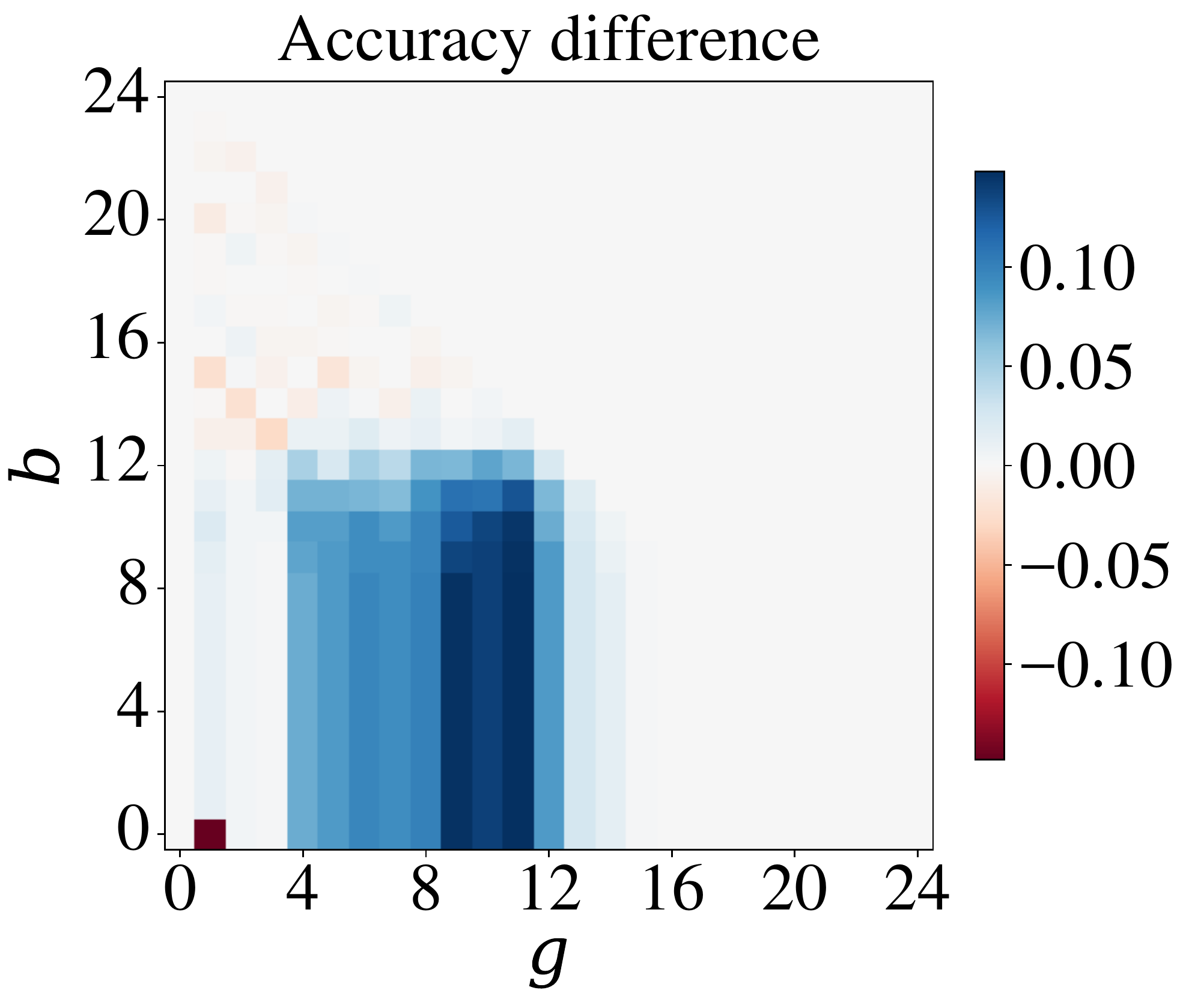}
    \includegraphics[height=4.5cm]{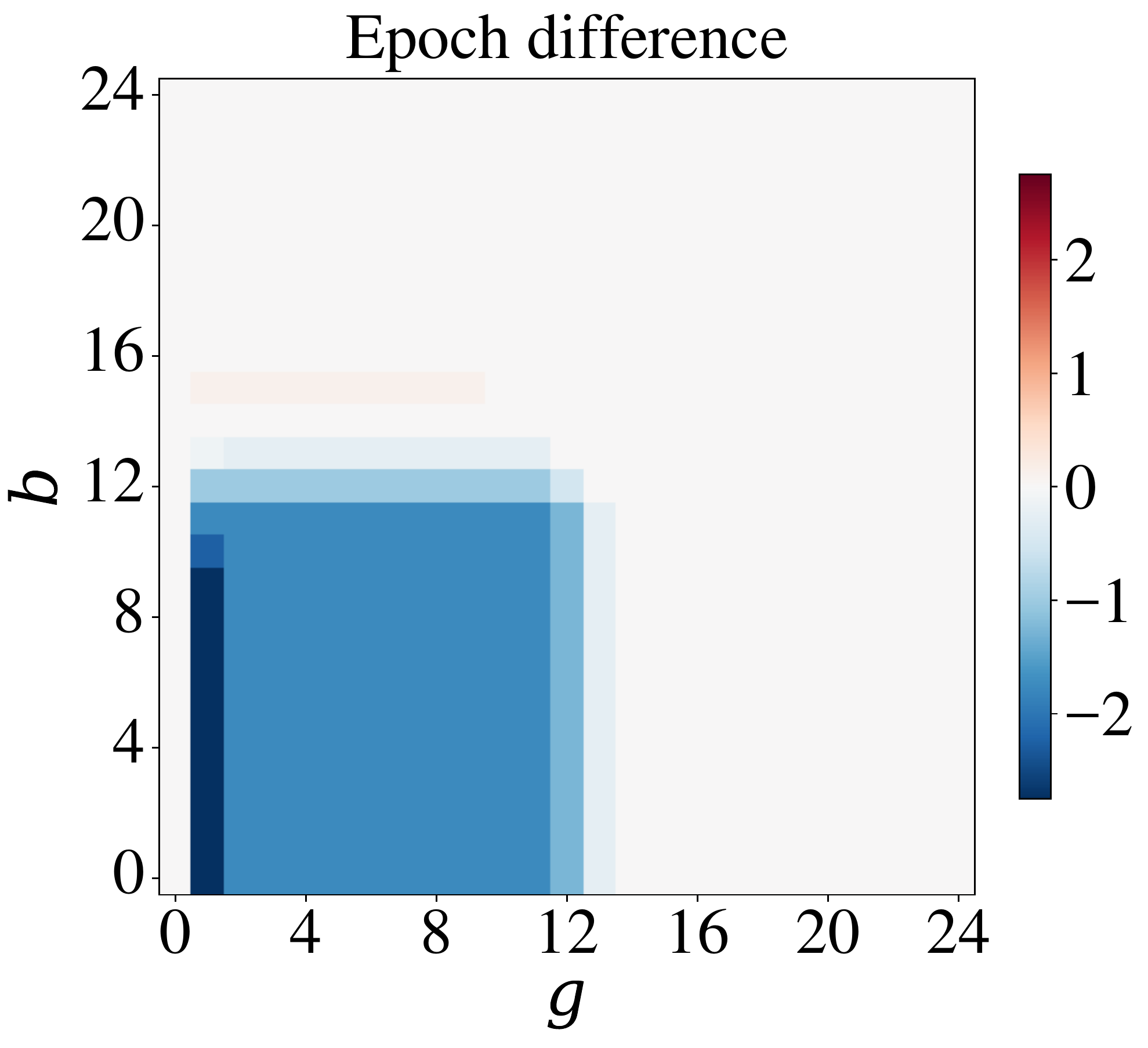}
  }
  \subcaptionbox{Number of triggers}{%
    \includegraphics[height=4.5cm]{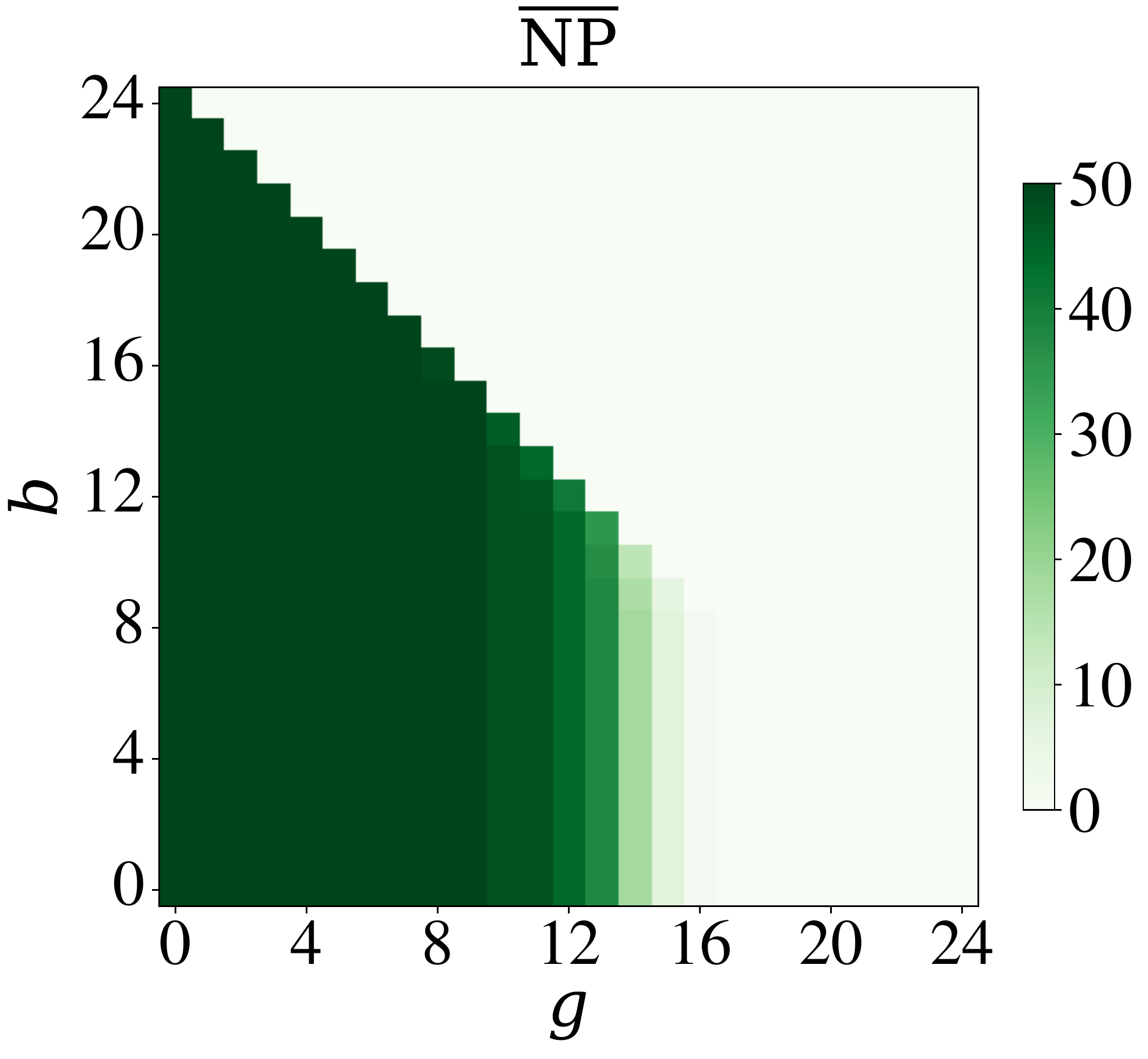}
    \includegraphics[height=4.5cm]{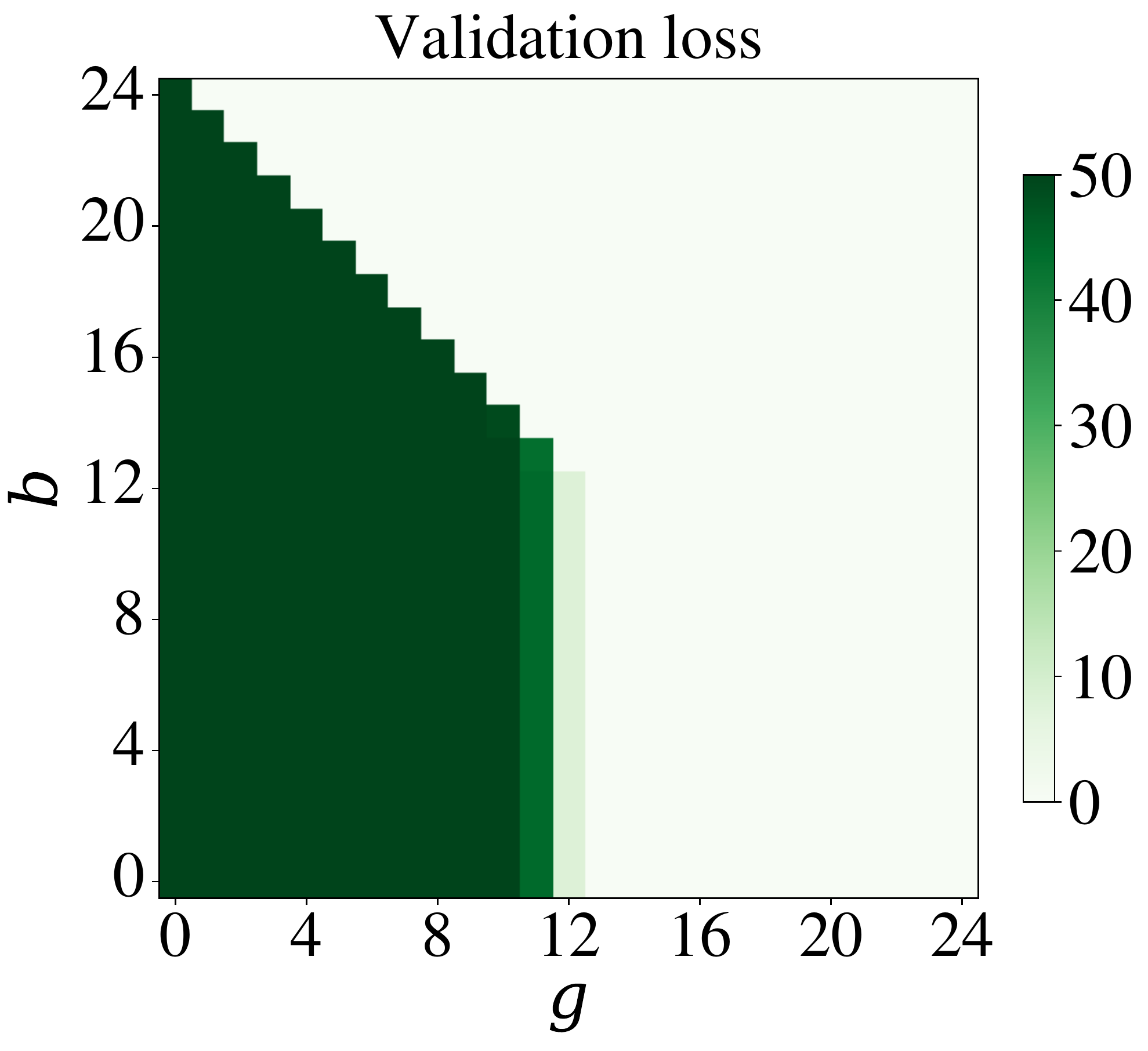}
  }
  \caption{%
    Additional visualizations for the `IMDB' data set.
  }
  \label{fig:IMDB additional}
\end{figure}

\begin{figure}[p]
  \centering
  \iffinal
    \includegraphics{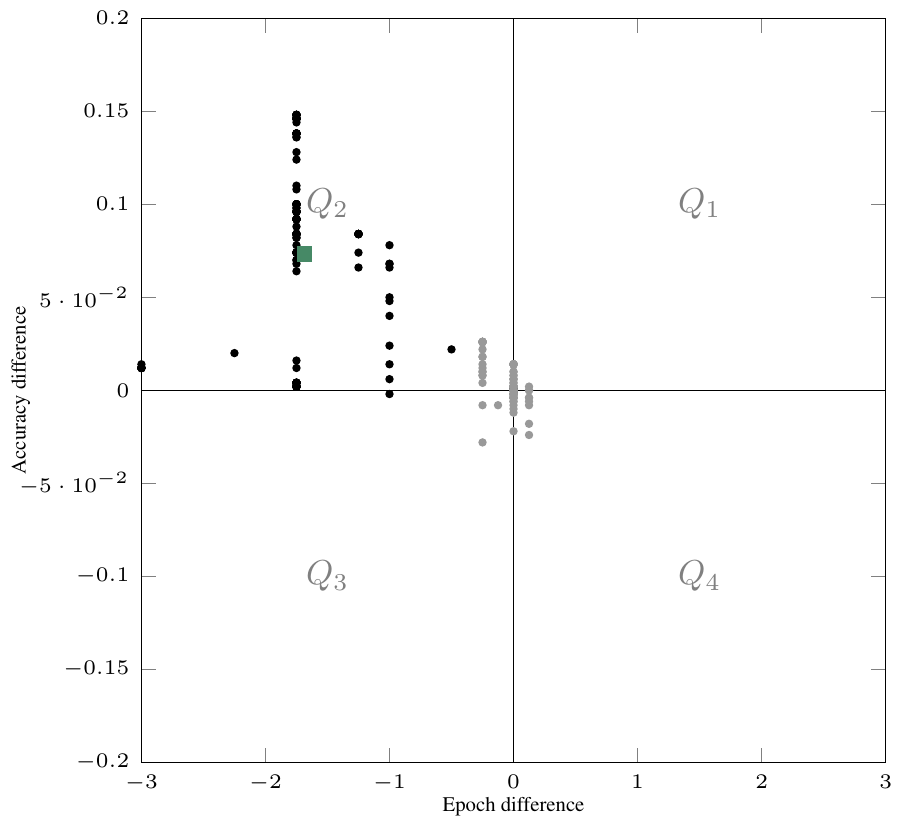}
  \else
    \input{figures/tikz/figure8.tex}
  \fi
  \caption{%
    Scatterplot of epoch and accuracy differences for `IMDB'.
  }
  \label{fig:IMDB scatterplot}
\end{figure}

\clearpage

\subsection{Neural Persistence for Convolutional Layers}\label{sec:cnn}

In principle, the proposed filtration process could be applied to any
bipartite graph. Hence, we can directly apply our framework to
convolutional layers, provided we represent them properly. 
Specifically, for layer $l$ we represent the convolution of its $i$th
input feature map $a_i^{(l-1)} \in \real^{h_{\textrm{in}} \times
w_{\textrm{in}} }$ with the $j$th filter $H_j \in \real ^{p \times q}
$ as one bipartite graph $G_{i,j}$ parametrized by a sparse weight
matrix $W_{i,j}^{(l)} \in \real^{ ( h_{\textrm{out}} \cdot w_{\textrm{out}}
) \times  ( h_{\textrm{in}} \cdot w_{\textrm{in}} ) }$, which in each
row contains the $p \cdot q$ unrolled values of $H_j$ on the diagonal,
with $h_{\textrm{in}} - p $ zeros padded in between after each $p$
values of $\textrm{vec}(H_j)$.
This way, the flattened pre-activation can be described as
$\textrm{vec}(z_{i,j}^{(l)}) = W_{i,j}^{(l)} \cdot
\textrm{vec}(a_{i}^{(l-1)})  + b_{i,j}^l \cdot
\mathbbm{1}_{(h_{\textrm{out}} \cdot w_{\textrm{out}}) \times 1} $. 

Since flattening does not change the topology of our bipartite graph, we
compute the normalized neural persistence on this sparse weight matrix
$W_{i,j}^{(l)}$ as the unrolled analogue of the fully-connected
network's weight matrix. Averaging over all filters then gives
a per-layer measure, similar to the way we derived mean normalized
neural persistence in the main paper.

When studying the unrolled adjacency matrix $W_{i,j}^{(l)}$, it becomes
clear that the edge filtration process can be approximated in a closed
form.  Specifically, for $m $ and $ n$ input and output neurons we
initialize $\tau = m + n $ connected components. When using zero padding,
the additional dummy input neurons have to included in $m$. For all
$\tau$ tuples in the persistence diagram the creation event $c=1$.
Notably, each output neuron shares the same set of edge weights.

Due to this, the destruction events---except for a few special
cases---simplify to a list of length $\tau$ containing the largest
filter values~(each value is contained $n$ times) in descending order
until the list is filled.
This simplification of neural persistence of a convolution with one
filter is shown as a closed expression in Equations~\ref{eq:conv_np1}--\ref{eq:conv_np3},
and our implementation is sketched in Algorithm~\ref{alg:NPconv}. We thus obtain
\begin{align}
  \neuralpersistence(G_{i,j}) &= \norm{\mathds{1} - \widetilde{\mathbf{w}} }_p,\label{eq:conv_np1}\\
\shortintertext{where we use}
  \norm{\widetilde{\mathbf{w}}}_p &\leq \norm{ \left( 0, \mathbf{w}_{c}^T, \mathbf{w}_{\bar{c}, \phi}^T, \textrm{vec}(A_{\phi})^T,  \textrm{vec}(B_{\phi})^T\ \right)^{T} }_p, \label{eq:conv_np2}\\
\shortintertext{with}
  \phi &= \tau-\textrm{dim}(\mathbf{w}_{c})-1,\\
  A_x  &= \mathbf{w}_{1:\floor*{\frac{x}{n}}} \otimes \mathds{1}_{n-1},\\
  B_y  &= \mathbf{w}_{\floor*{\frac{y}{n}} +1 } \otimes \mathds{1}_{y \ \textrm{mod} \ n},\label{eq:conv_np3}
\end{align}%
where $\mathds{1}_{0} \coloneqq 0$.
Following this notation, Equation~\ref{eq:conv_np1} expresses neural
persistence of the bipartite graph $G_{i,j}$, with $\widetilde{\mathbf{w}}$
denoting the vector of selected weights~(i.e.\ the destruction events)
when calculating the persistence diagram.
We use $\mathbf{w}$ to denote the flattened and sorted weight values~(in
descending order) of the convolutional filter $H_j$, while $\mathbf{w}_{c}$
represents the vector of all weights that are located in a corner of
$H_j$, whereas $\mathbf{w}_{\bar{c}, \phi}$ is the vector of all weights
which do \emph{not} originate from the corner of the filter while still
belonging to the first~(and thus \emph{largest}) $\floor*{\frac{\phi}{n}}$
weights in $\mathbf{w}$, which we denote by $\mathbf{w}_{1:\floor*{\frac{\phi}{n}}}$.

For the subsequent experiments~(see below), we use a simple CNN that
employs $32 + 2048$ filters. Hence, by using the shortcut described above,
we do not have to unroll 2080 weight matrices explicitly, thereby
gaining \emph{both} in memory efficiency and run time, as compared to
the naive approach: on average, a naive exact computation based on
unrolling required \SI{8.77}{\second} per convolutional filter and
evaluation step, whereas the approximation only took about
\SI{0.00038}{\second} while showing very similar behaviour up to
a constant offset.

\begin{algorithm}[tbp]
	\caption{Approximating Neural Persistence of Convolutions per filter}
	\label{alg:NPconv}
	\algorithmicrequire{} filter $H \in \real^{p \times q} $; number of input and output neurons as $m,  n$
	\begin{algorithmic}[1]
		\State  $\mathcal{T} \gets \emptyset$ 									\Comment{Initialize set of tuples for persistence diagram}
		\State  $\tau \gets m+n, \ \ t \gets 0,  \ \ i \gets 0 $ \Comment{Initialize number of tuples, tuple counter, weight index}  		
		\State $h_{\max} \gets \max_{h \in H} |h|$                     \Comment{Determine largest absolute weight}
		
		\State $H' \gets \{ |h| / h_{\max} \mid h \in H \}$           \Comment{Transform weights for filtration} \label{pcline:c_init}
		
		\State $s \gets \textrm{sort}(\textrm{vec}(H')) $ \Comment{Sort weights in descending order}
		\State $H'_c \gets \{h'_{0,0}, h'_{0,q-1}, h'_{p-1,0}, h'_{p-1,q-1}\} $  \Comment{Determine the set of all corner weights of filter $H'$}
		\State $\mathcal{T} \gets (1, 0 ) , \ \ t \gets t+1 $   \Comment{Add tuple for surviving component}
		\For{$h'_{c} \in H'_c $} \Comment{Each corner of $H'$ merges components}
		\State {$\mathcal{T} \gets (1, h'_{\textrm{c}} ) , \ \ t \gets t+1 $  } 
		\EndFor

		\While{$ 1 $} \Comment{Create the remaining tuples (Approximation step)}
		\State $n' = n - \textrm{Ind}(s[i] \in H'_c)$ \Comment{if current weight is a corner weight, write one less tuple}
		\If{$t + n' \leq \tau $}  \Comment{if there are at least $n'$ more tuples, set their merge value to $s[i] $ }
		\State{\textbf{repeat} $n'$ times} 
		\State\hspace{\algorithmicindent}{$\mathcal{T} \gets (1, s[i]) $ } \Comment{approximative as $s[i]$ does not always add $n'$ merges due to loops}
		\State{$ t \gets t + n', \ \  i \gets i+1 $}
		\Else \Comment{otherwise, process the remaining tuples similarly }
		\State{\textbf{repeat} $(\tau-t)$ times} 
		\State\hspace{\algorithmicindent}{$\mathcal{T} \gets (1, s[i]) $ } 
		\State{\textbf{break}}
		\EndIf{}
		\EndWhile{}

		\State \textbf{return} $ \norm{\mathcal{T}}_p  $  \Comment{Compute norm of approximated persistence diagram}
	\end{algorithmic}
\end{algorithm}


For our experiments, we used an off-the-shelf `LeNet-like' CNN model
architecture~(two convolutional layers each with max pooling and ReLU,
1 fully-connected and softmax) as described in \citet{tensorflowConv15}.
We trained the model on `Fashion-MNIST' and included this setup in the
early stopping experiments~(100 runs of 20 epochs).
In Figure~\ref{fig:conv additional}, we observe that stopping based on
the neural persistence of a convolutional layer typically \emph{only}
incurs a considerable loss of accuracy: given a final test accuracy of
$91.73\pm0.13$, stopping with this naive extension of our measure
reduces accuracy by up to 4\%.
Furthermore, in contrast to early stopping on a fully-connected
architecture, we do not observe any parameter combinations that stop
early \emph{and} increase accuracy. In fact, there is no configuration
that results in an increased accuracy. This empirically confirms our
theoretical scepticism towards naively applying our edge-focused
filtration scheme to CNNs.

\begin{figure}[tbp]
  \centering
  \subcaptionbox{Accuracy and epoch differences}{
    \includegraphics[height=4cm]{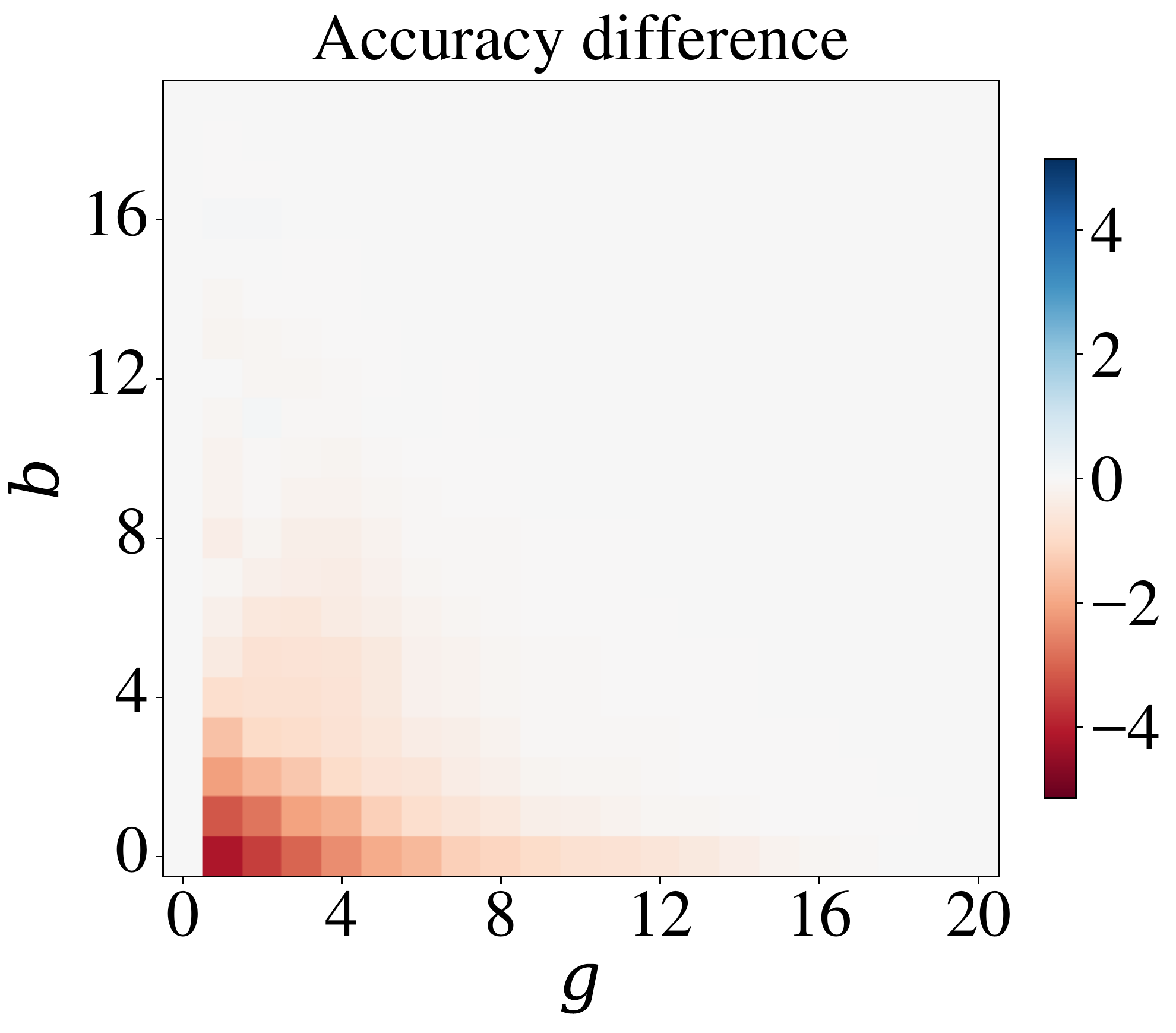}
    \includegraphics[height=4cm]{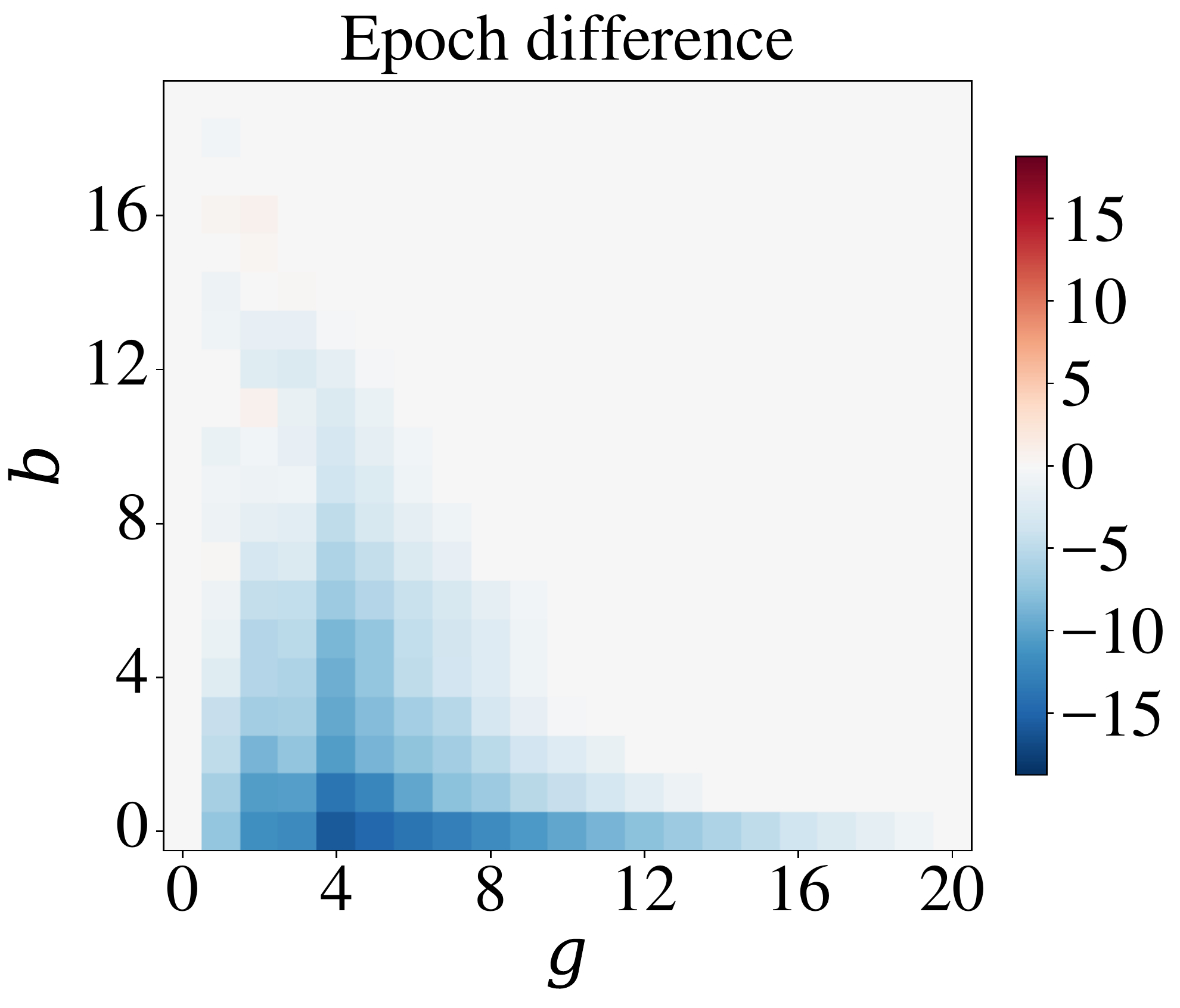}
  }
  \subcaptionbox{Number of triggers}{
    \includegraphics[height=4cm]{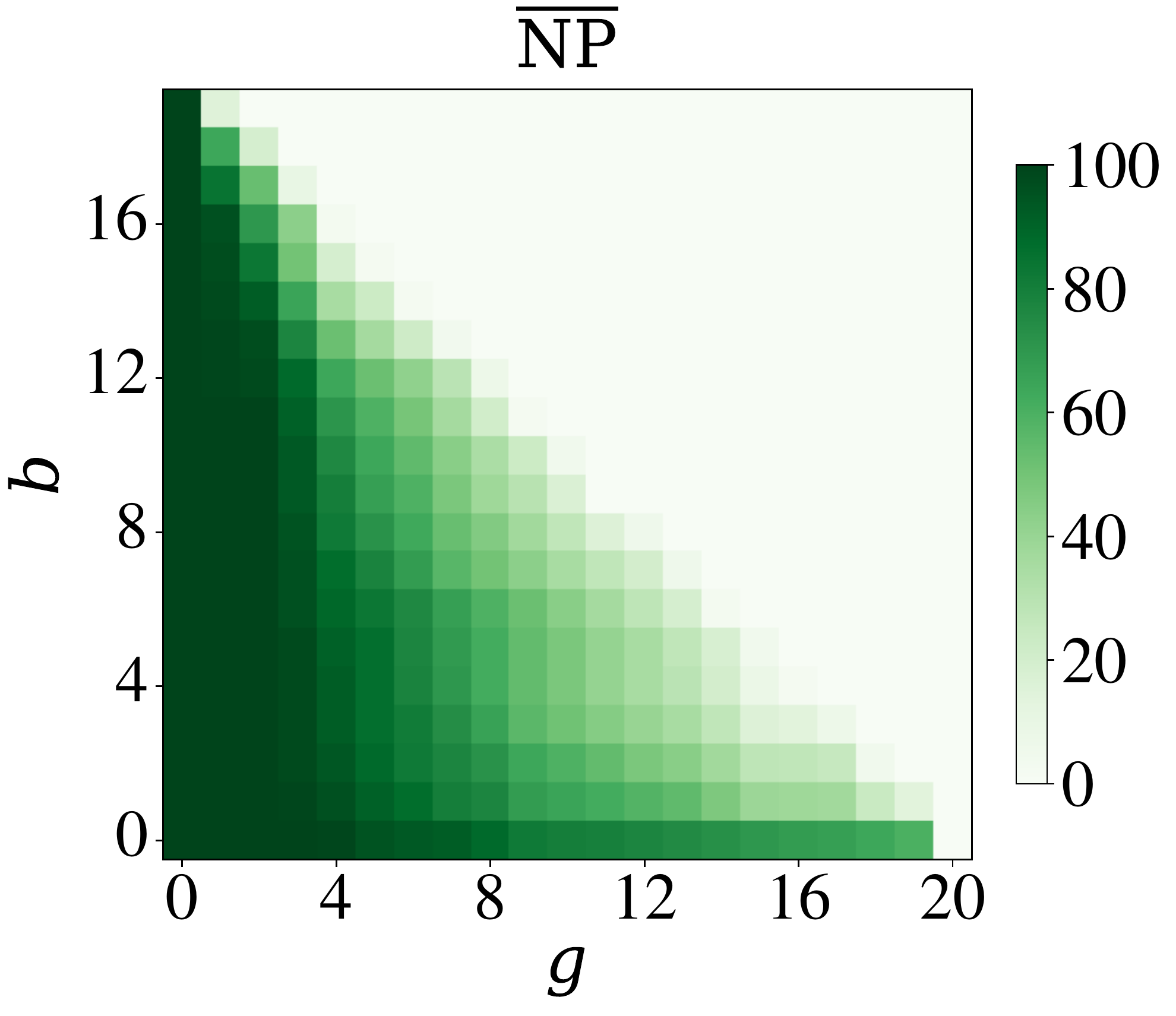}
    \includegraphics[height=4cm]{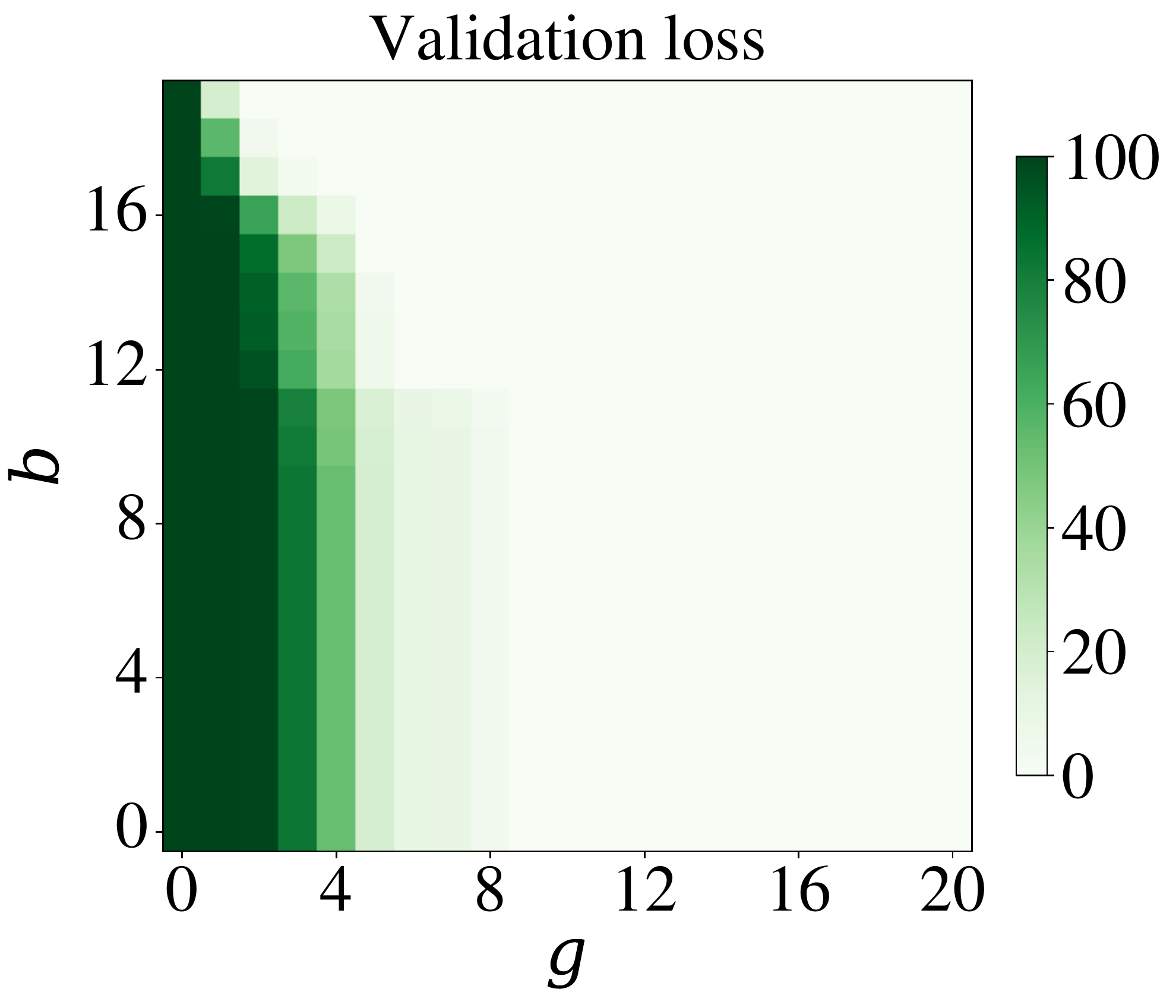}
  }
  \caption{%
    Additional visualizations for the `Fashion-MNIST' data set,
    following the preliminary examination of convolutional layers.
    Here, the approximated neural persistence calculation for the first
    convolutional layer was used. However, we also ran few runs of the
    same experiment using the exact method which showed the same
    results. Employing the second convolutional layer or both did not
    improve this result.
  }
  \label{fig:conv additional}
\end{figure}

\subsection{Relationship between neural persistence and validation accuracy}\label{sec:val}

Motivated by Figure~\ref{fig:Neural persistence regimes}, which shows
the different `regimes' of neural persistence for a perceptron network,
we investigate a possible correlation of~(high) neural persistence
with~(high) predictive accuracy.
For deeper networks, we find that neural persistence measures
structural properties that arise from different parameters~(such as training
procedures or initializations), and \emph{no} correlation can be observed.

For our experiments, we constructed neural networks
with a \emph{high} neural persistence prior to training.
More precisely, following the theorems in this paper, we initialized most weights
of each layer with very low values and reserved high values for very few weights.
This was achieved by sampling the weights from a \emph{beta distribution} with
$\alpha=0.005$ and $\beta=0.5$.
Using this procedure, we are able to initialize~[20,20,20]
networks with $\meanneuralpersistence \approx 0.90 \pm 0.003$ compared to the
same networks that have $\meanneuralpersistence \approx 0.38 \pm 0.004$ when
initialized by Xavier initialization.
The mean validation accuracy of these untrained networks on the
`Fashion-MNIST' data set is $0.10 \pm 0.01$ and $0.09 \pm 0.03$,
respectively.

Figure~\ref{fig:np init} depicts how both types of networks converge to
similar regimes of validation accuracy, while the mean normalized neural
persistence achieved at the end of the training varies. For networks initialized with high $\meanneuralpersistence$~(Figure~\ref{fig:np init}, left) the validation accuracy of networks with final $0.9 \leq
\meanneuralpersistence \leq 0.95$ ranges from $0.098$~(not shown) to $0.863$.
For Xavier initialization~(Figure~\ref{fig:np init}, right), the lack of correlation can also be observed. Furthermore, comparing the two plots, there are no clear advantages in initializing networks with high $\meanneuralpersistence$. This observation further motivates the proposed \emph{early stopping criterion}, which checks for
\emph{changes} in the $\meanneuralpersistence$ value, and considers
stagnating values to be indicative of a trained network.

\begin{figure}[tbp]
  \centering
  \includegraphics[width=\textwidth]{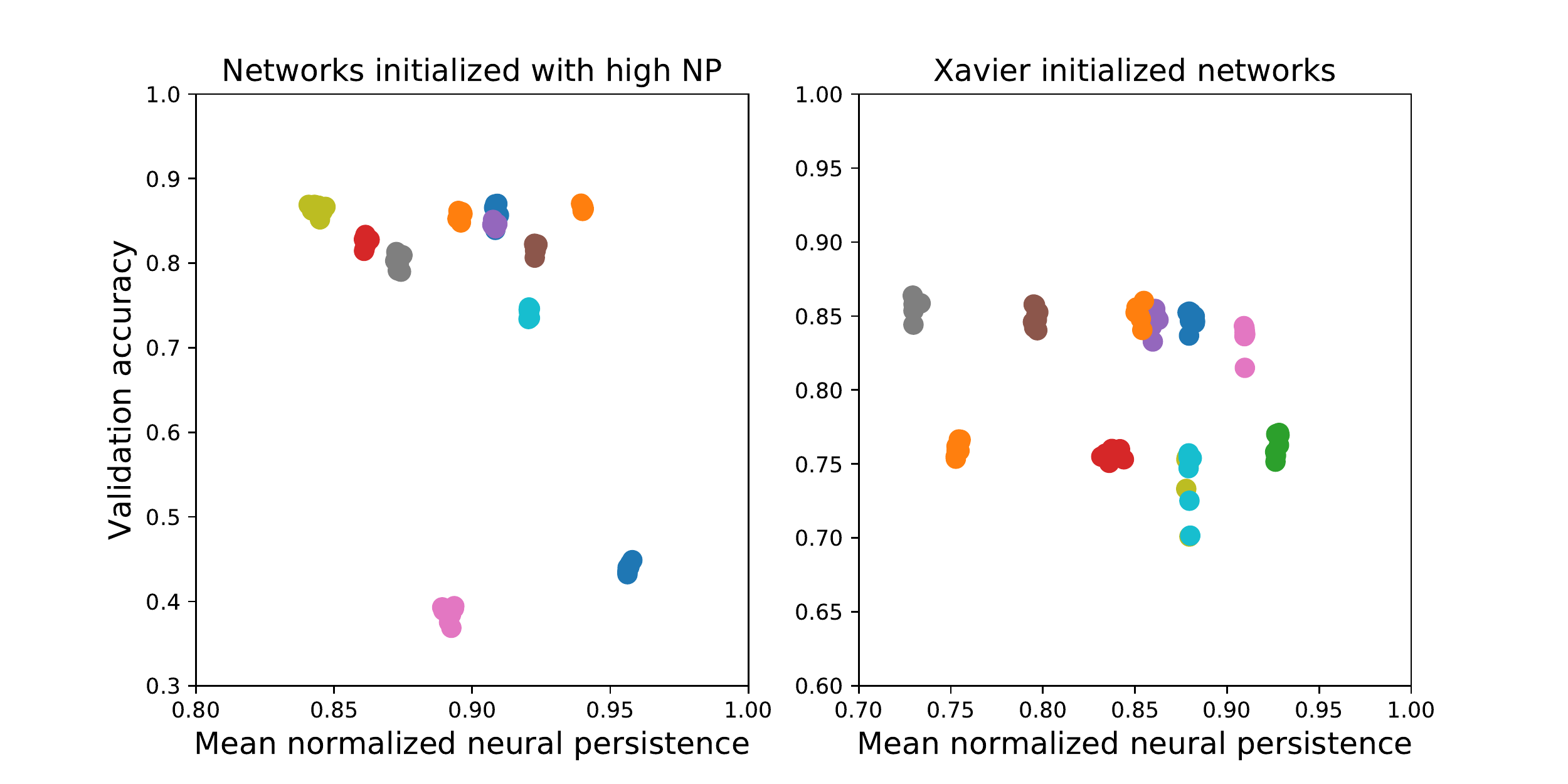}
  \caption{%
    Each cluster of points represent the last two training
    epochs~(sampled every quarter epoch) of a [20,20,20] network trained
    on the `Fashion-MNIST' data set.
    We observe no correlation between validation accuracy and normalized
    total persistence
  }
  \label{fig:np init}
\end{figure}

\clearpage

\subsection{Neural persistence for different data distributions and deeper FCN architectures}\label{sec:datadistrib}

Neural persistence captures information about different data
distributions during training. The weights tuned via backpropagation are
directly influenced by the input data~(as well as their labels) and
neural persistence tracks those changes.
To demonstrate this, we trained the same architecture~, i.e.\ $[50,50,20]$, on two
data sets with the same dimensions but different properties: MNIST and
`Fashion-MNIST'.
Each data set has the same image size~($28 \times  28$ pixels, one channel)
but lay on different manifolds.
Figure~\ref{fig:np_data_distr_and_deep} (left) shows a histogram of the mean normalized
neural persistence~($\meanneuralpersistence$) after $25$ epochs of training
over $100$ different runs.
The distributions have a similar shape but are shifted, indicating that
the two datasets lead the network to different topological regimes.

We also investigated the effect of depth on neural persistence. We
selected a fixed layer size~(20 hidden units) and increased the number
of hidden layers. Figure~\ref{fig:np_data_distr_and_deep} (right) depicts the boxplots of
mean $\meanneuralpersistence$ for multiple architectures
after 15 epochs of training on MNIST.
Adding layers initially increases the variability of
$\meanneuralpersistence$ by enabling the network to converge to
different regimes~(essentially, there are many more valid configurations
in which a trained neural network might end up in).
However, this effect is reduced after a certain depth: networks with
deeper architectures exhibit less variability in $\meanneuralpersistence$.

\begin{figure}[tbp]
  \centering
  \includegraphics[width=0.49\textwidth]{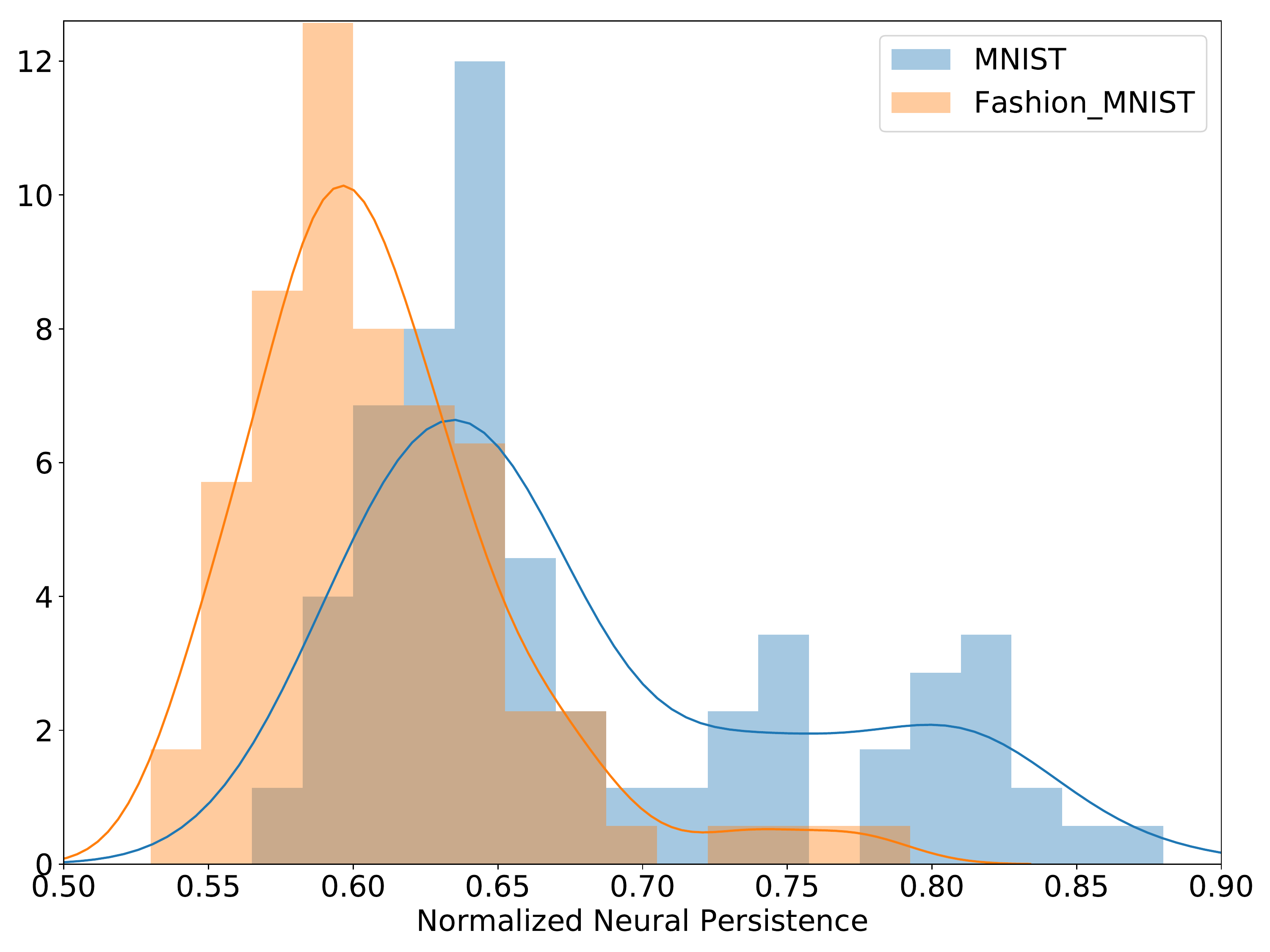}
  \includegraphics[width=0.49\textwidth]{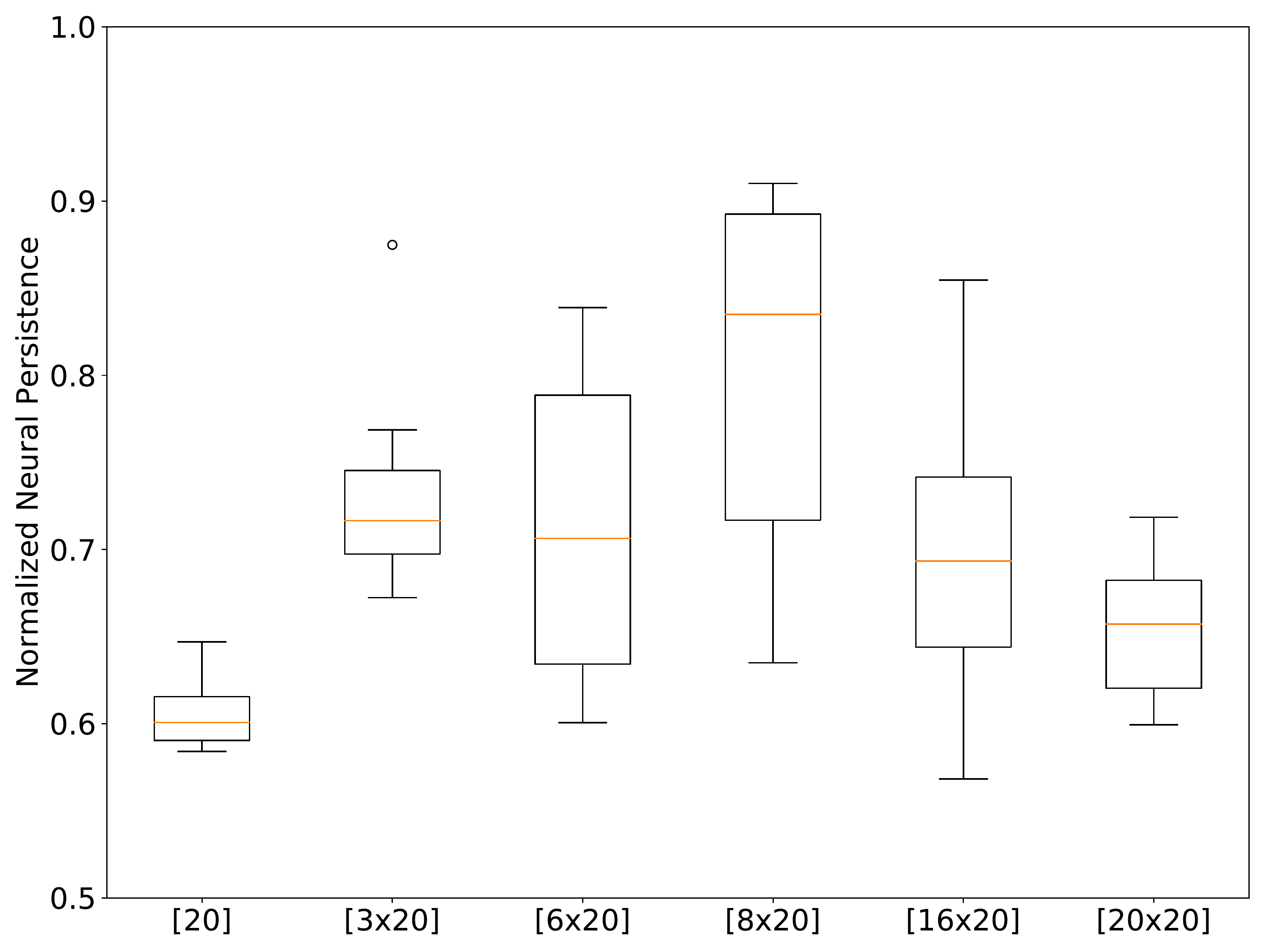}
  \caption{(left) Histogram of the final normalized neural persistence of a $[50,50,20]$ network for 100 runs and 25 epochs of training. (right) Normalized neural persistence after 15 epochs of training on MNIST for different architectures with increasing depth. Deeper architectures are denoted as $[n\times20]$ where $n$ is the number of hidden layers.}
  \label{fig:np_data_distr_and_deep}
\end{figure}

\clearpage

\subsection{Early stopping in data scarcity scenarios}\label{sec:extreme}

Labelled data is expensive in most domains of interest, which results in small data sets or low quality of the labels.
We investigate the following experimental set-ups: 
\begin{inparaenum}[(1)]
\item Reducing the training data set size and
\item Permuting a fraction of the training labels.
\end{inparaenum}
We train a fully connected network ($[500,500,200]$ architecture) on `MNIST' and `Fashion-MNIST'.
In the experiments, we compare the following measures for stopping the training:
\begin{inparaenum}[i)]
\item Stopping at the optimal test accuracy.
\item Fixed stopping after the burn in period.
\item Neural persistence patience criterion.
\item Training loss patience criterion.
\item Validation loss patience criterion.
\end{inparaenum}
For a description of the patience criterion, see Algorithm~\ref{alg:Early stopping}. All measures, except validation loss, include the validation datasets ($20 \%$) in the training process to simulate a larger data set when no cross-validation is required. We report the accuracy on the non-reduced, non-permuted test sets. The batch size is $32$ training instances. The stopping measures are evaluated every quarter epoch.

\begin{figure}[p]
\centering
\includegraphics[width=\textwidth]{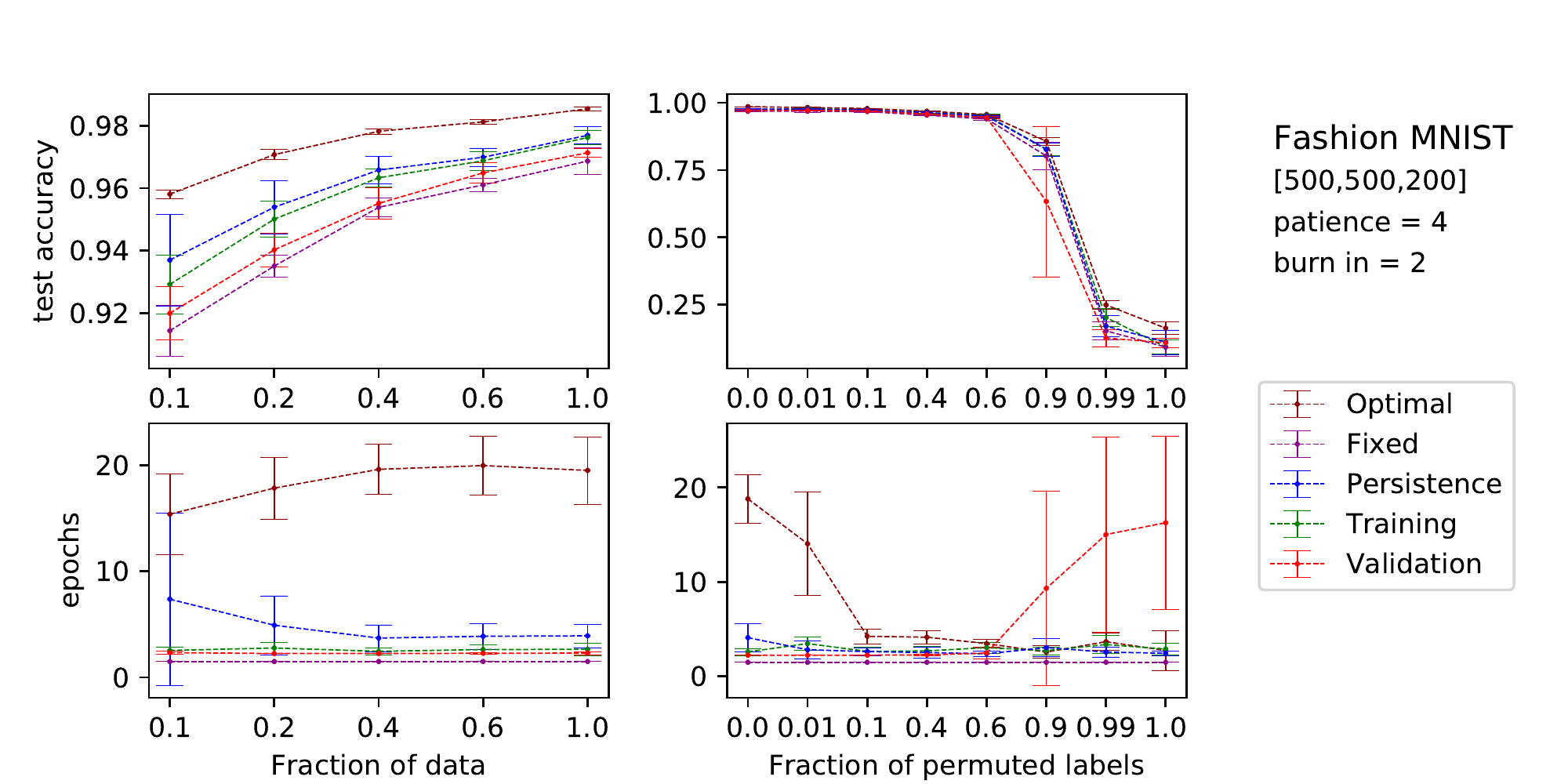}
\includegraphics[width=\textwidth]{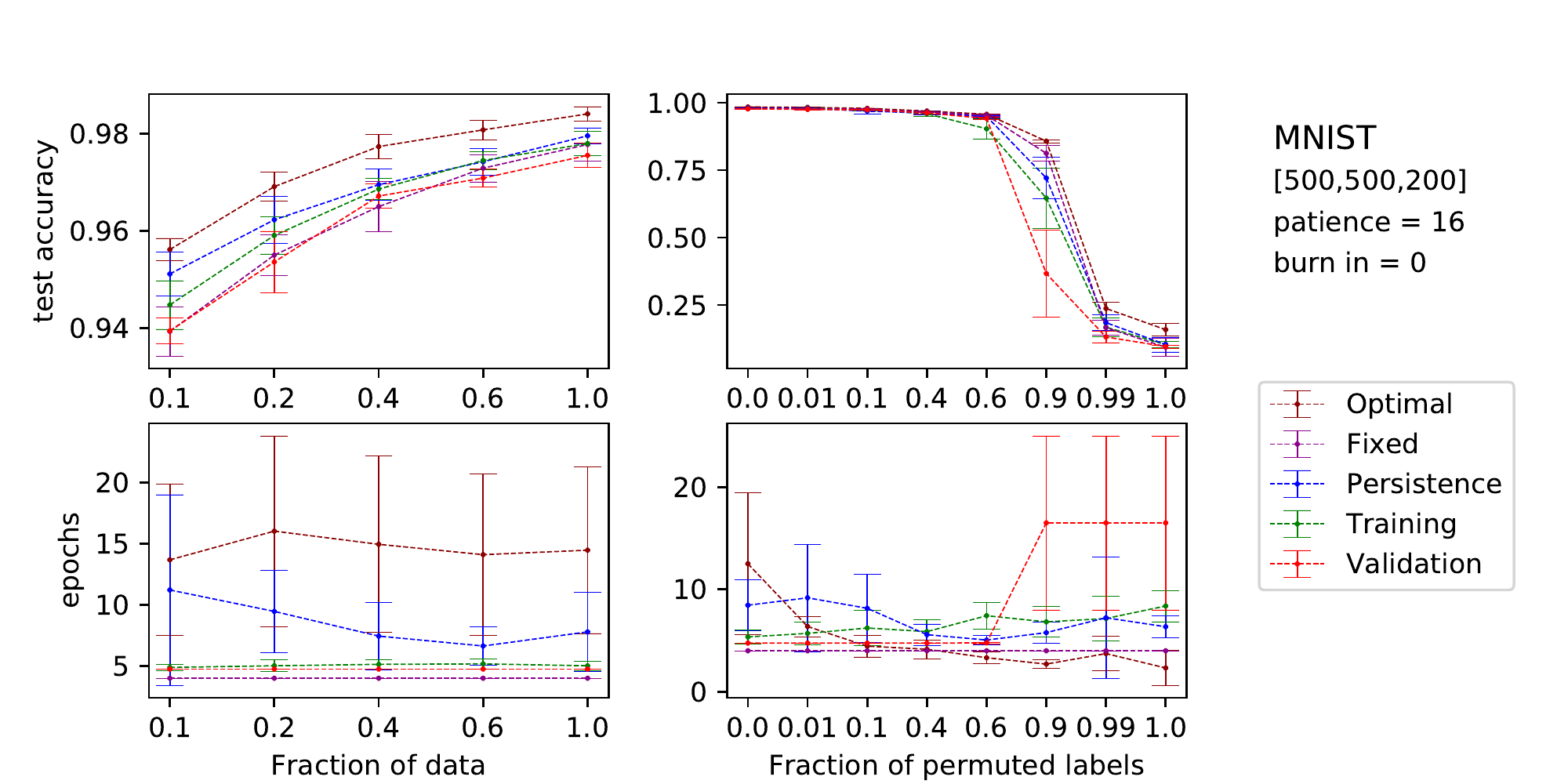}
\caption{On MNIST and Fashion-MNIST $\meanneuralpersistence$ (in blue) stops later than validation and training loss when fewer training samples are available (left-hand side) which results in a higher test accuracy. For increasing noise in the training labels (right-hand side), the stopping of $\meanneuralpersistence$ remains stable, in contrast to the validation loss stopping, which leads to lower test accuracy after longer training at a high fraction of permuted labels. The patience and burn in parameters are reported in quarter epochs.}\label{fig:Fraction}
\end{figure}

Figure \ref{fig:Fraction} shows the results averaged over $10$ runs (the error is the standard deviation). 
The difference between the top and the bottom panel is the data set and the patience parameters.
The $x$-axis depicts the fraction of the data set, which is warped for better accessibility. 
In each panel, the left-hand side subplots depict the results of the reduced data set experiment where the right-hand side subplots depict the result of the permutation experiments.
The $y$-axis of the top subplot shows the accuracy on the non-reduced, non-permuted test set. The $y$-axis of the bottom subplot shows when the stopping criterion was triggered.

We note the following observations, which hold for both panels:
More, non-permuted data yields higher test accuracy. Also, as expected, the optimal stopping gives the highest test accuracy.
The fixed early stopping results in inferior test accuracy when only a fraction of the data is available.
The neural persistence based stopping is triggered late when only a fraction of the data is available which results in a slightly better test accuracy compared to training and validation loss.
The training loss stopping achieves similar test accuracies compared to the persistence based stopping (for all regimes except the very small data set) with shorter training, on average.
We note that, it is generally not advisable to use training loss as a measure for stopping because the stability of this criterion also depends on the batch size.
When only a fraction of the data is available, the validation loss based stopping stops on average after the same number of training epochs as the training loss, which results in inferior test accuracy because the network has seen in total fewer training samples. 
Most strikingly, validation loss based stopping is is triggered later (sometimes never) when most training and validation labels are randomly permuted which results in overfitting and poor test accuracy.

To conclude, the neural persistence based stopping achieves good performance without being affected by the batch size and noisy labels. The authors also note that the result is consistent for multiple architectures and most patience parameters.


\begin{sidewaystable}[tbp]
		\caption{%
			Parameters and hyperparameters for the experiment on best practices and neural persistence. Dropout and batch normalization were applied after the first hidden layer. Throughout the networks, \emph{ReLU} was the activation function of choice.
		}
		\label{tab:best}
		\begin{tabular}{lllllll}
			\toprule
			Data set				& \# Runs & \# Epochs & Architecture & Optimizer                   & Batch Size & Hyperparameters  \\
			\midrule
	\multirow{3}{*}{MNIST}    & \multirow{3}{*}{50}	& \multirow{3}{*}{40} & \multirow{3}{*}{$[650, 650]$} & \multirow{3}{*}{Adam} & \multirow{3}{*}{32} & $\eta = 0.0003$ $\beta_1 = 0.9$, $\beta_2 = 0.999$, $\epsilon = \num{1e-8}$\\
			& 	& 	&  &  &  & $\eta = 0.0003$ $\beta_1 = 0.9$, $\beta_2 = 0.999$, $\epsilon = \num{1e-8}$, Batch Normalization\\
			& 	& 	&  &  &  & $\eta = 0.0003$ $\beta_1 = 0.9$, $\beta_2 = 0.999$, $\epsilon = \num{1e-8}$, Dropout 50\%\\
	    
			\bottomrule
		\end{tabular}
	
	\bigskip\bigskip  


	  \caption{%
	    Parameters and hyperparameters for the experiment on early stopping. Throughout the networks, \emph{ReLU} was the activation function of choice.
	  }
	  \label{tab:early}
	  \begin{tabular}{lllllll}
	    \toprule
	    Data set				& \# Runs & \# Epochs & Architecture & Optimizer                   & Batch Size & Hyperparameters  \\
	    \midrule
	    \multirow{3}{*}{(Fashion-)MNIST} & \multirow{3}{*}{100}	& 10	& Perceptron & Minibatch SGD & 100 & $\eta = 0.5$\\
	    & 	& \multirow{3}{*}{40} & $[50,50,20]$ & \multirow{3}{*}{Adam} & \multirow{3}{*}{32} & \multirow{3}{*}{$\eta = 0.0003$ $\beta_1 = 0.9$, $\beta_2 = 0.999$, $\epsilon = \num{1e-8}$}\\
	    & 	& 	& $[300,100]$ &  &  & \\
	    & 	& 	& $[20,20,20]$ &  &  & \\
	    \midrule
	    CIFAR-10 & 10 & 80 & $[800,300,800]$ & Adam & 128 & $\eta = 0.0003$ $\beta_1 = 0.9$, $\beta_2 = 0.999$, $\epsilon = \num{1e-8}$ \\
      \midrule
	    IMDB & 5 & 25 & $[128,64,16]$ & Adam & 128 & $\eta = \num{1e-5}$ $\beta_1 = 0.9$, $\beta_2 = 0.999$, $\epsilon = \num{1e-8}$ \\
	    \bottomrule
	  \end{tabular}
\end{sidewaystable}

\clearpage


\subsection{Testing accuracy of differently regularized models}

We showed in the main text that neural persistence is capable of
distinguishing between networks trained with/without batch normalization
and/or dropout. Figure~\ref{fig:folkwisdom} additionally shows test set
accuracies.

\begin{figure}[tbp]
  \centering
  \iffinal
    \includegraphics{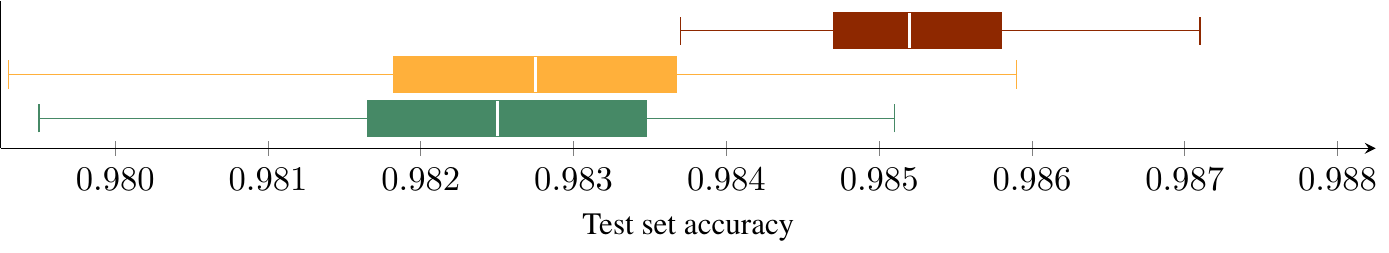}
  \else
    \input{figures/tikz/figure9.tex}
  \fi
  \caption{%
     Comparison of test set accuracy for trained
     networks without modifications~(\textcolor{printable_1}{green}), with batch
     normalization~(\textcolor{printable_2}{yellow}), and with 50\% of the neurons
     dropped out during training~(\textcolor{printable_3}{red}) for the MNIST data
     set.
  }
  \label{fig:folkwisdom}
\end{figure}

\end{document}